\documentclass[pdflatex,sn-basic]{sn-jnl}

\usepackage{physics}
\usepackage{comment}
\usepackage{booktabs}

\usepackage{xcolor,colortbl} 
\usepackage{amsmath} 
\usepackage{multirow}
\usepackage{amsthm}
\usepackage{algorithm}
\usepackage{algpseudocode}
\usepackage{graphicx}%
\usepackage{multirow}%
\usepackage{amsmath,amssymb,amsfonts}%
\usepackage{amsthm}%
\usepackage{mathrsfs}%
\usepackage[title]{appendix}%
\usepackage{xcolor}%
\usepackage{textcomp}%
\usepackage{manyfoot}%
\usepackage{booktabs}%
\usepackage{algorithm}
\usepackage{algpseudocode}
\usepackage{algorithmicx}%
\usepackage{listings}%
\usepackage{graphicx}

\newtheorem{problem}{Problem}
\theoremstyle{thmstyleone}%
\newtheorem{theorem}{Theorem}
\theoremstyle{thmstyletwo}%

\theoremstyle{thmstylethree}%
\newtheorem{definition}{Definition}%

\begin{document}

\title[Automated Unsupervised Graph Anomaly Detection]{Towards Automated Self-Supervised Learning for Truly Unsupervised Graph Anomaly Detection\footnote{Manuscript accepted by \textit{Data Mining and Knowledge Discovery} journal for publication (June 2025). This is a preprint version.}}


\author*{\fnm{Zhong} \sur{Li}}\email{z.li@liacs.leidenuniv.nl}

\author{\fnm{Yuhang} \sur{Wang}}\email{yuhang.wang@umail.leidenuniv.nl}

\author{\fnm{Matthijs} \sur{van Leeuwen}}\email{
m.van.leeuwen@liacs.leidenuniv.nl}

\affil{\orgdiv{Leiden Institute of Advanced Computer Science (LIACS)}, \orgname{Leiden University}, \country{the Netherlands}}

\abstract{Self-supervised learning (SSL) is an emerging paradigm that exploits supervisory signals generated from the data itself, and many recent studies have leveraged SSL to conduct graph anomaly detection. However, we empirically found that three important factors can substantially impact detection performance across datasets: 1) the specific SSL strategy employed; 2) the tuning of the strategy's hyperparameters; and 3) the allocation of combination weights when using multiple strategies. Most SSL-based graph anomaly detection methods circumvent these issues by arbitrarily or selectively (i.e., guided by label information) choosing SSL strategies, hyperparameter settings, and combination weights. While an arbitrary choice may lead to subpar performance, using label information in an unsupervised setting is label information leakage and leads to severe overestimation of a method's performance. Leakage has been criticized as ``one of the top ten data mining mistakes", yet many recent studies on SSL-based graph anomaly detection have been using label information to select hyperparameters. To mitigate this issue, we propose to use an internal evaluation strategy (with theoretical analysis) to select  hyperparameters in SSL for unsupervised anomaly detection.  We perform extensive experiments using 10 recent SSL-based graph anomaly detection algorithms on various benchmark datasets, demonstrating both the prior issues with hyperparameter selection and  the effectiveness of our proposed strategy.}


\maketitle

\section{Introduction}

Graph anomaly detection (GAD) refers to the tasks of identifying anomalous graph objects---such as nodes, edges or sub-graphs---in an individual graph \citep{akoglu2015graph, ma2021comprehensive}, or identifying anomalous graphs from a set of graphs \citep{ma2022deep, li2024cross}. GAD has numerous successful applications, e.g., in finance fraud detection \citep{motie2023financial}, fake news detection \citep{xu2022evidence},  system fault diagnosis \citep{li2024graph}, and network intrusion detection \citep{garcia2009anomaly}. In this paper, we focus on unsupervised node anomaly detection on static attributed graphs, namely identifying which nodes in a static attributed graph are anomalous. Recently, Graph Neural Networks (GNNs) have become prevalent in detecting node anomalies in graphs and have shown promising performance \citep{kim2022graph}. Specifically, GNNs can learn an embedding for each node by considering both the node attributes and the graph topological information, enabling them to capture and exploit complex patterns for anomaly detection. 

Like with other neural networks, the high performance of GNNs is typically achieved at the cost of a substantial volume of labeled data. However, the process of labeling graphs is often a laborious and time-consuming effort, necessitating domain-specific expertise. For these reasons, GAD is preferably tackled in an unsupervised manner, without relying on any ground-truth labels. Self-supervised learning (SSL) has emerged as a promising unsupervised learning technique on graphs \citep{liu2022graph}, and recent studies have shown its usefulness for node anomaly detection \citep{fan2020anomalydae,zheng2021generative,jin2021anemone,liu2021anomaly,yuan2021higher,xu2022contrastive,liu2022dagad,chen2022gccad}. 
 

Graph SSL can be roughly divided into \textit{generative}, \textit{contrastive}, and \textit{predictive} methods \citep{wu2021self}. First, \textit{generative} methods such as DOMINANT \citep{ding2019deep}, GUIDE \citep{yuan2021higher}, and AnomalyDAE \citep{fan2020anomalydae} aim to detect graph anomalies by reconstructing (`generating') the adjacency matrix and/or the node attribute matrix. Next, \textit{contrastive} methods such as CoLA \citep{liu2021anomaly}, ANEMONE \citep{jin2021anemone}, GRADATE \citep{duan2023graph}, and Sub-CR \citep{zhang2022reconstruction} train a graph encoder to pull positive pairs closer while pushing negative pairs away in the embedding space. The nodes with relatively large contrastive loss values are deemed anomalies. Finally, \textit{predictive} methods such as SL-GAD \citep{zheng2021generative} try to predict node properties using its local context (e.g., a subgraph), and nodes with large prediction errors are considered anomalies. 

Contrastive learning is arguably the most successful SSL strategy for graphs \citep{xie2022self}. Most contrastive graph learning methods consist of two main modules: 1) a \textit{data augmentation module} that generates augmented data by operations such as edge dropping, node attribute masking, node addition, subgraph sampling, and/or graph diffusion. The augmented view of an instance is generally regarded as a positive pair with the original instance; and 2) a \textit{contrastive learning module} that contrasts positive pairs (and often involves negative pairs) at different levels, such as node-node contrast, node-subgraph contrast, and subgraph-subgraph contrast. 

Although SSL-based graph anomaly detection has been successful, using it in practice is often not straightforward. The most important reason for this is that most methods require a large number of choices to be made, leading to three challenges: 
\begin{itemize}
    \item[\textbf{C1.}] How should we select appropriate data augmentation functions?
    \item[\textbf{C2.}] How should we choose appropriate values for hyperparameters (HPs) of a given augmentation function? (e.g., subgraph size in a subgraph sampling function, or the proportion of edges to drop in an edge dropping function)
    \item[\textbf{C3.}] How to combine the contrast losses at different levels? (i.e., how to set their combination weights?)
\end{itemize}
Further, a recent study \citep{zheng2021generative} shows that combining multiple SSL strategies for GAD can achieve better performance than using a single SSL strategy. This leads to the fourth challenge: 
\begin{itemize}
    \item[\textbf{C4.}]How should we combine different SSL strategies?(i.e., how to set the combination weights of different SSL loss functions?)
\end{itemize}

Previous work \citep{chen2020simple,you2020graph, yoo2023dsv} showed that the choice of SSL strategies and hyperparameter values can strongly impact performance. In a \emph{supervised} setting, these choices can be systematically and rigorously made by using separate labeled data for validation. In an \emph{unsupervised setting} such as anomaly detection, however, one should assume that no labels are available even for hyperparameter tuning. In our extensive literature study, we found that existing SSL-based GAD methods typically either 1) arbitrarily choose settings or 2) do use labeled data, corroborating the findings in \cite{yoo2023dsv}. 

In the former case, practitioners typically heuristically select an augmentation function (C1) and fix its associated HPs (C2) across all datasets, and set the combination weights all equal to 1 or other fixed values (for C3 and C4). Although this approach is not flawed, it is likely to result in suboptimal detection performance: graphs from different domains usually have different properties \citep{zhao2022autogda}, implying that the optimal SSL strategy is in general data-dependent \citep{chen2020simple,you2020graph}. Therefore, utilizing a unified and pre-fixed combination weights and/or  HPs in SSL strategies for all graphs can result in sub-optimal performance.

In the latter case, practitioners pick the optimal combination weights and other hyperparameter values following a `hyperparameters sensitivity analysis' using labeled data. By using ground-truth labels on test data to check model performance with different hyperparameter values and using that to select the best model, however, \emph{label leakage} occurs. That is, information about the target of a data mining problem is used for learning/selecting model, while this information should not be legitimately accessible for learning purposes \citep{nisbet2009handbook,kaufman2012leakage}. Specifically, label information should never be used (whether implicitly or explicitly) in an unsupervised learning scenario. As shown in Figure~\ref{Fig:ROC_Var}, label leakage leads to huge overestimation of the model's performance, which is also corroborated in \cite{liu2022bond} by comparing the max and average performance with different hyperparameter configurations (cf. Appendix~\ref{app:SimilarObservations} for more details). 

\begin{figure}
\centering
\includegraphics[width=13cm]{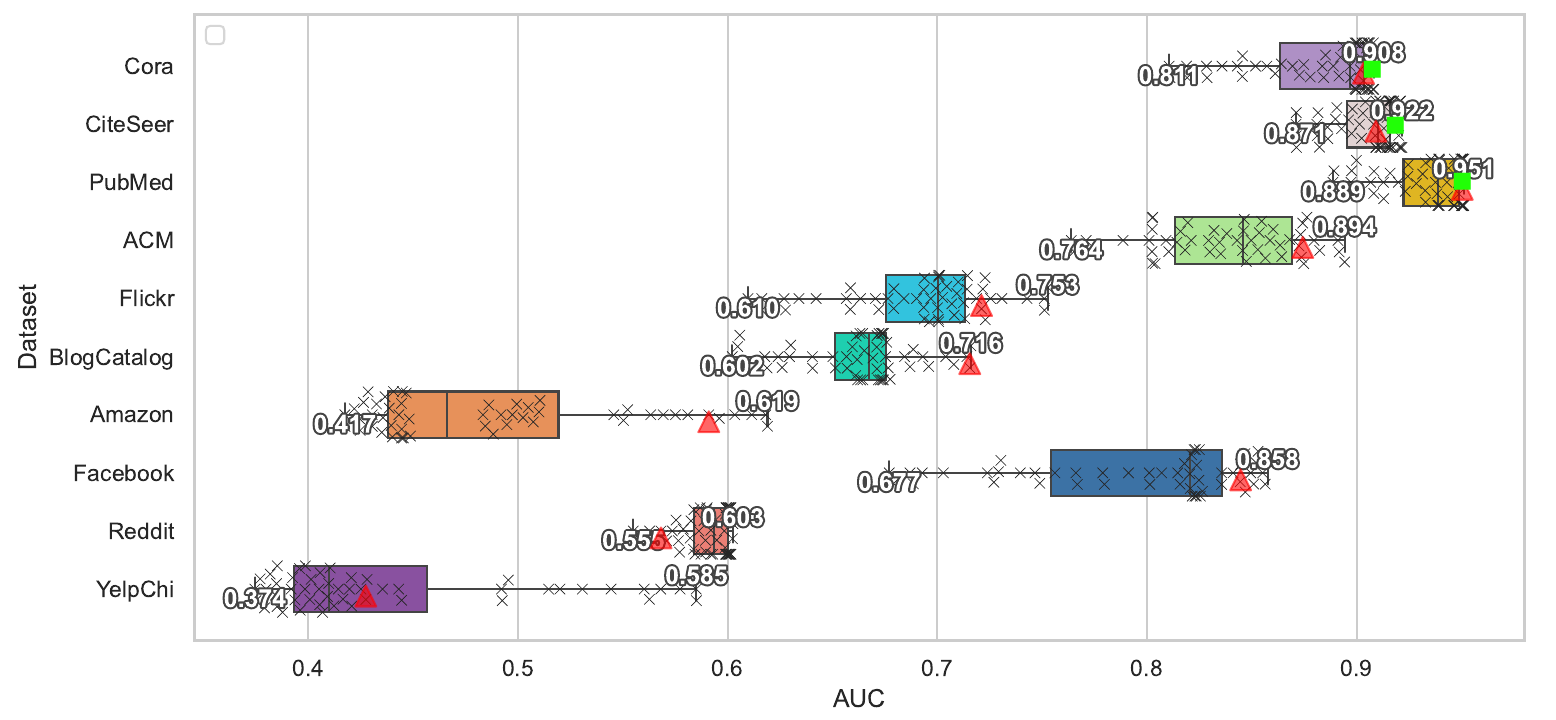}
\caption{Large performance variations (here measured by AUC) over different hyperparameter configurations for ANEMONE~\citep{jin2021anemone} on various benchmark datasets. Using labeled data and only reporting the best possible performance leads to severe overestimation of model performance. For instance, the green squares on Cora, CiteSeer, and PubMed are reported by \cite{jin2021anemone} (the other datasets were not used). Similar results are observed for other algorithms (see Appendix~\ref{appendix:FurtherAnalysis} for details). The red triangles represent the results obtained by our internal evaluation strategy, showing its potential for automating truly unsupervised anomaly detection.}
\label{Fig:ROC_Var}
\end{figure}

The reason that hyperparameter values are often chosen either arbitrarily or using label information is probably that it is challenging to construct an internal evaluation strategy for anomaly detection without using any labels.
There have been some research efforts aimed at automating graph SSL though. For instance, JOAO \citep{you2021graph} aims to automatically combine several predefined graph augmentations via learning a sampling distribution, where the augmentations themselves are not learnable. Meanwhile, AD-GCL \citep{suresh2021adversarial} uses  learnable edge dropping augmentation and AutoGCL \citep{yin2022autogcl} proposes a learnable graph view generator that learns a probability distribution over node-level augmentations, which can well preserve the semantic labels of graphs for graph-level tasks. However, all these automated graph augmentation methods are agnostic to the downstream tasks, making the learned graph embeddings sub-optimal for a specific downstream task, namely anomaly detection in our case. Additionally, these methods are specifically designed for certain SSL frameworks, and it is non-trivial (if at all possible) to extend them to the general SSL framework. Moreover, these automated SSL strategies are computationally expensive, rendering them impractical in real-world applications.

As an initial step towards mitigating this long-standing but neglected issue, we propose a lightweight and plug-and-play approach dubbed \textit{AutoGAD}, to automate SSL for truly unsupervised graph anomaly detection. Specifically, \textit{AutoGAD} leverages a so-called internal evaluation strategy \citep{ma2023need}, without relying on any ground-truth labels (whether explicitly or implicitly), to select optimal combination weights and/or SSL-specific hyperparameter values. Moreover, we theoretically analyze the internal evaluation strategy to prove why it is effective and empirically demonstrate this. 

Overall, our main contributions can be summarized as follows:
\begin{itemize}
    \item We raise renewed awareness to the \textit{label information leakage} issue, which is critical but often overlooked in the unsupervised GAD field;
    \item Although there exists a plethora of graph SSL methods and GAD approaches, we are the first to investigate automated SSL specifically for unsupervised GAD;
    \item We propose a lightweight, plug-and-play approach to automate SSL for truly unsupervised GAD and provide a theoretical analysis;
    \item Extensive experiments are conducted using 10 state-of-the-art SSL-based GAD algorithms on {\color{black} 10} benchmark datasets, demonstrating the effectiveness of our approach.
\end{itemize}

\section{Related Work}
Our work is related to node anomaly detection on static attributed graphs, self-supervised learning for graph anomaly detection, automated self-supervised learning, and automated anomaly detection.

\subsection{Anomaly Detection on Attributed Graphs}
Early methods for node anomaly detection in static attributed graphs, such as AMEN \citep{perozzi2016scalable}, Radar \citep{li2017radar}, and Anomalous \citep{peng2018anomalous}, are not based on deep learning. These methods work well on low-dimensional attributed graphs, but their performance is limited on complex graphs with high-dimensional node attributes.

Recently, deep learning-based  methods, including DOMINANT \citep{ding2019deep}, AnomalyDAE \citep{fan2020anomalydae}, and GUIDE \citep{yuan2021higher}, have been proposed for GAD. These methods usually employ graph autoencoders to encode nodes followed by decoders to reconstruct the adjacency matrix and/or node attributes. As a result, nodes with large reconstruction errors are considered anomalies. 
Despite their superior performance to non-deep learning methods, these reconstruction-based methods still suffer from sub-optimal performance, as reconstruction is a generic unsupervised learning objective. Besides, these methods require the full attribute and adjacency matrices as model input, making them unsuitable or even impossible for large graphs.

\subsection{Self-Supervised Learning for Graph Anomaly Detection}
Graph SSL aims to learn a model by using supervision signals generated from the graph itself, without relying on human-annotated labels \citep{liu2022graph}. It has achieved promising performance on typical graph mining tasks such as representation learning \citep{jiao2020sub} and graph classification \citep{zeng2021contrastive}.  \cite{liu2021anomaly} first applied SSL to the GAD problem. Their proposed method CoLA performs single scale comparison (node-subgraph) for anomaly detection. However, ANEMONE \citep{jin2021anemone} argues that modeling the relationships in a single contrastive perspective leads to limited capability of capturing complex anomalous patterns. Hence, they propose additional node-node contrast.
Additionally, GRADATE \citep{duan2023graph} and M-MAG \citep{liu2023revisiting} combines various multi-contrast objectives, namely node-node, node-subgraph, and subgraph-subgraph contrasts for node anomaly detection. To achieve better performance, SL-GAD \citep{zheng2021generative} combines multi-view contrastive learning and generative attribute regression, while Sub-CR \citep{zhang2022reconstruction} combines multi-view contrastive learning and graph autoencoder. Finally, CONAD \citep{xu2022contrastive} considers both contrastive learning and generative reconstruction for better node anomaly detection.

\subsection{Automated Self-Supervised Learning}

Seminal work on \emph{automated data augmentation for images} \citep{ratner2017learning,cubuk2018autoaugment} was followed by work improving \citep{cubuk2018autoaugment} via faster searching mechanisms \citep{ho2019population,lim2019fast, cubuk2020randaugment} or advanced optimization methods \citep{hataya2020faster, li2020differentiable, zhang2019adversarial}.

In the context of \emph{automated data augmentation for graphs}, related work exists on graph representation learning \citep{hassani2022learning,suresh2021adversarial, jin2021automated, xie2022self,yin2022autogcl,you2021graph}, node classification \citep{zhao2021data,sun2021automated}, and graph-level classification \citep{luo2022automated,yue2022label,yin2022autogcl}. For example, JOAO \citep{you2021graph} learns the sampling distribution of a set of predefined graph augmentations. AD-GCL \citep{suresh2021adversarial} designs a learnable edge dropping augmentation and employs adversarial training strategy, and AutoGCL \citep{yin2022autogcl} proposes a learnable graph view generator that learns a probability distribution over the node-level augmentations. Further, \cite{luo2022automated} augment graph data samples, while \cite{yue2022label} perturb the representation vector. However, these methods focus on other typical graph learning tasks and it is unclear how to use them for unsupervised GAD.

\subsection{Automated Anomaly Detection}

Recent studies \citep{zhao2021automatic,bahri2022automl,ding2022hyperparameter, zhao2022towards} pointed out that unsupervised anomaly detection methods tend to be highly sensitive to the values of their hyperparameters (HPs). For example, \cite{zhao2021automatic} show that a 10x performance difference is observed for LOF \citep{breunig2000lof} by changing the number of nearest neighbors. Even more, \cite{ding2022hyperparameter} indicate that deep anomaly detection methods suffer more from such HP sensitivity issues. Concretely, \cite{zhao2022towards} demonstrate that RAE \citep{zhou2017anomaly} exhibits a 37x performance difference with different HPs configurations. 

To tackle this issue, automated HP tuning and model selection for unsupervised anomaly detection has received increasing but insufficient attention;  \cite{bahri2022automl} present an overview. Inspired by \cite{bahri2022automl,zhao2022towards}, we subdivide existing approaches into two main categories: 
\begin{itemize}
    \item \emph{Supervised evaluation} methods which require ground-truth labels although anomaly detection algorithms are unsupervised. Methods include PyODDS \citep{li2020pyodds}, TODS \citep{lai2021tods}, AutoOD \citep{li2021autood}, and AutoAD \citep{li2021automated};
    \item \emph{Unsupervised evaluation} methods which do not require ground-truth labels. They include 
    \begin{itemize}
        \item randomly selecting an HP configuration;
        \item selecting an HP configuration via an internal evaluation strategy \citep{goix2016evaluate, zhao2019lscp, marques2020internal, putina2022autoad};
        \item averaging the outputs of a set of randomly selected HP configurations \citep{wenzel2020hyperparameter};
        \item meta-learning based methods \citep{zhao2020automating, zha2020meta, zhao2022towards}.
    \end{itemize}
\end{itemize}
However, existing automated anomaly detection methods are primarily designed for non-graph data.

\section{Problem Statement}
We utilize lowercase letters, bold lowercase letters, uppercase letters, and calligraphic fonts to represent scalars ($x$), vectors ($\mathbf{x}$), matrices ($\mathbf{X}$), and sets ($\mathcal{X}$), respectively. 

\begin{definition}[Attributed Graph] We denote an attributed graph as $\mathcal{G} = \{\mathcal{V}, \mathcal{E}, \mathbf{X}\}$, where $\mathcal{V} = \{v_{1},...,v_{n}\}$ is the set of nodes. Besides, $\mathcal{E} = \{e_{ij}\}_{i,j \in\{1,...,n\}}$ is the set of edges, where $e_{ij}=1$ if there exists an edge between $v_{i}$ and $v_{j}$ and $e_{ij}=0$ otherwise. Moreover, $\mathbf{X} \in \mathbb{R}^{n\times d}$ represents the node attribute matrix, where the $i$-th row vector $\mathbf{x}_{i}$ means the node attribute of $v_{i}$.
\end{definition}

Formally, we consider unsupervised node anomaly detection on attributed graphs (dubbed GAD hereafter), which is defined as follows:

\begin{problem}[Node Anomaly Detection on Attributed Graph] Given an attributed graph as $\mathcal{G} = \{\mathcal{V}, \mathcal{E}, \mathbf{X}\}$, we aim to learn an anomaly scoring function $f(\cdot)$ that assigns an anomaly score $s = f(v_{i})$ to each node $v_{i}$, with a higher score representing a higher degree of being anomalous. Next, the anomaly scores are used to rank the nodes such that the top-$k$ nodes can be considered as anomalies. 
\end{problem}

 In this paper, we consider the \textit{transductive unsupervised anomaly detection} setting:  the graph containing both normal and abnormal nodes is given at the training stage. Node labels are not accessible during the training stage and they are only used for performance evaluation. Importantly, the labels of nodes are \textbf{not} (and should not be) used for HP tuning under this unsupervised setting.
 
Formally, we consider the hyperparameter optimization problem for unsupervised graph anomaly detection (dubbed HPO for GAD):
\begin{problem}[HPO for GAD]
Given a graph  $\mathcal{G}$ without labels and a graph anomaly detection algorithm $f(\cdot)$ with hyperparameter space $\mathbf{\Lambda}$, we aim to identify a hyperparameter configuration  $\boldsymbol{\lambda} \in \mathbf{\Lambda}$ such that the resulting model $f({\boldsymbol{\lambda}})$ can achieve the best performance on $\mathcal{G}$. I.e., suppose $\boldsymbol{\lambda}$ consists of $K$ different hyperparameters $\{\mathbf{\lambda}_{1},...,\mathbf{\lambda}_{k},...,\lambda_{K}\}$, where  $\mathbf{\lambda}_{k} \in \mathbf{\Lambda}_{k}$ can be discrete or continuous, we then aim to find
\begin{equation}
\label{Equ:HPO}
  \underset{\mathbf{\lambda}_{1} \in \mathbf{\Lambda}_{1},...,\mathbf{\lambda}_{k} \in \mathbf{\Lambda}_{k},...,\mathbf{\lambda}_{K} \in \mathbf{\Lambda}_{K}}{\arg\max} \text{Metric}\left[f(\mathbf{\lambda}_{1},...,\mathbf{\lambda}_{k},...,\mathbf{\lambda}_{K};\mathcal{G})\right],  
\end{equation}
where $Metric[\cdot]$ is a given performance metric.
\end{problem}

\section{SSL for Unsupervised GAD}
In this section, we first revisit existing self-supervised learning methods for ``unsupervised" graph anomaly detection, followed by an analysis and experiments to showcase pitfalls in existing studies.

\subsection{Existing SSL for ``Unsupervised" GAD}

\begin{figure}[]
\centering
\includegraphics[width=13cm]{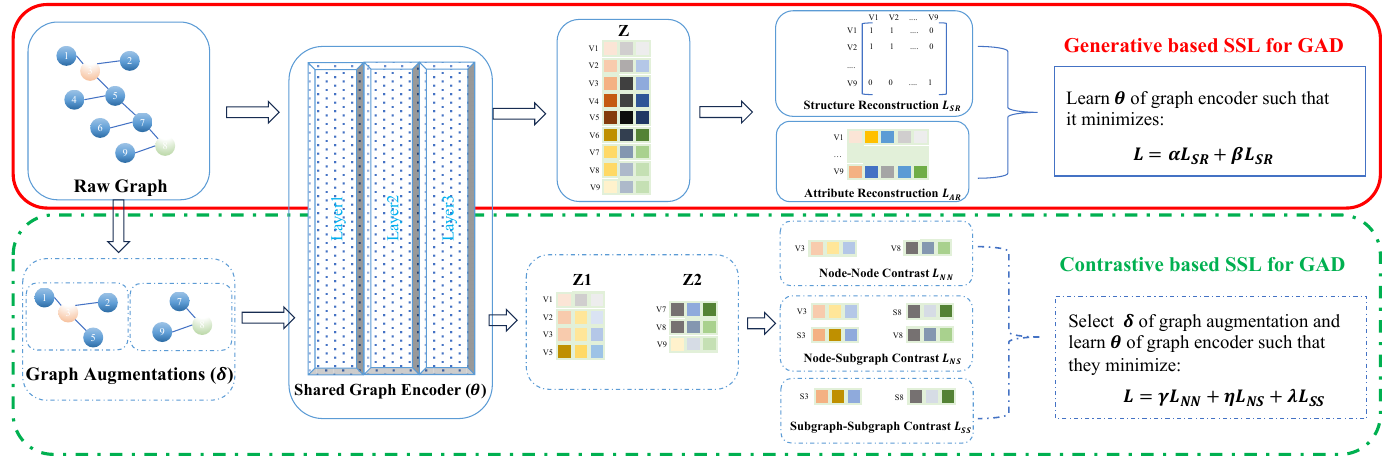}
\caption{Self-supervised learning based graph anomaly detection methods can be subdivided into \textit{generative based} methods and \textit{contrastive based} methods. A \textit{generative based} method generally involves \textit{graph structure reconstruction} and \textit{node attributes reconstruction}. A \textit{contrastive based} method usually consists of a \textit{graph augmentation module} and a \textit{contrastive learning module}.}
\label{Fig:AutoGAD_WF}
\end{figure}

Figure~\ref{Fig:AutoGAD_WF} shows how existing SSL based GAD methods can be divided into \textit{generative} methods and \textit{contrastive} methods. 

That is, a \textit{generative} method usually consists of two individual SSL tasks, namely 1.1) \textit{structure reconstruction} that aims to reconstruct the adjacency matrix, and 1.2) \textit{attribute reconstruction} that aims to reconstruct the node attribute matrix. On this basis, the attribute reconstruction error and the structure reconstruction error are combined to obtain an anomaly score, where higher reconstruction error indicates a higher degree of anomalousness.

Meanwhile, a \textit{contrastive} method often consists of two modules: 2.1) \textit{data augmentation} module, and 2.2) \textit{contrastive learning} module. First, for each target node, the \textit{data augmentation} module utilizes one augmentation function $f(\boldsymbol{\delta})$ to produce augmented samples, which usually include positive samples and negative samples. The scenario of using multiple augmentation functions can be obtained in a similar way. Second, three contrastive perspectives can be applied to contrast positive pairs and negative pairs: 2.2.1) \textit{node-node contrast} that contrasts node embedding with node embedding, and  2.2.2) \textit{node-subgraph contrast} that contrasts node embedding with subgraph embedding, and  2.2.3) \textit{subgraph-subgraph contrast} that contrasts subgraph embedding with subgraph embedding.

\subsection{Pitfalls in Existing Methods}
\label{subsec:pitfalls}
In this subsection, we revisit existing SSL-based unsupervised GAD methods by checking the following three aspects for each method:
\begin{itemize}
    \item Which SSL framework does the method employ: \textit{generative}, \textit{contrastive}, or both?
    \item How many SSL-specific hyperparameters are involved? (E.g., combination weights and others.)
    \item How are values for key SSL hyperparameters chosen? (E.g., the ratio of node attribute masking or dropping edges, and the combination weights of multiple loss functions?)
\end{itemize}

By doing so, we point out that these studies have noticeable pitfalls. More importantly, we perform experiments to show that the high performance that these methods claim to achieve is often strongly overestimated due to label leakage issues (cf. Table~\ref{tab:performance_variation}). 

Due to space constraints and to enhance readability, we revisit three representative SSL-based GAD algorithms in the main paper, including a contrastive method: ANEMONE \citep{jin2021anemone}, a generative method: AnomalyDAE \citep{fan2020anomalydae}, and a combined contrastive and generative method: SL-GAD \citep{zheng2021generative}.

\subsubsection{Revisiting ANEMONE} ANEMONE \citep{jin2021anemone} is a \textit{contrastive} method for unsupervised GAD. 

\textbf{Graph Augmentation Module.} A single graph augmentation operation is used, namely Random Ego-Nets generation with a fixed size $K$. Specifically, taking the target node as the center, they employ RWR \citep{tong2006fast} to generate two different subgraphs as ego-nets with a fixed size $K$. This results in one critical HP, namely $K$.

\textbf{Contrast Learning Module.} Two contrast perspectives are considered: 1) node-node contrast between the embedding of a masked target node within the ego-net and the embedding of the original node, leading to loss term $\mathcal{L}_{NN}$, and 2) node-subgraph contrast within each view, leading to loss term $\mathcal{L}_{NS}$. These loss terms are combined as $\mathcal{L} = (1-\alpha)\mathcal{L}_{NN}+\alpha\mathcal{L}_{NS}$, where $\alpha \in [0,1]$ is the trade-off HP, giving one more critical HP, namely $\alpha$.

\textbf{HPs Sensitivity \& Tuning.} By using ground-truth label information, they heuristically set $\alpha$ to $0.8, 0.6, 0.8$ on Cora, CiterSeer, and PubMed respectively, and report the corresponding results. The setting of $K$ is not studied, and is set to 4 for all datasets.

\subsubsection{Revisiting AnomalyDAE } AnomalyDAE \citep{fan2020anomalydae} is a \textit{generative} method using autoencoders (based on GNNs) for unsupervised GAD.

\textbf{Generative Framework.} AnomalyDAE consists of two components: 1) an attribute autoencoder to reconstruct the node attributes, where the encoder consists of two non-linear feature transform
layers and the decoder is simply a dot product operation. This leads to the loss term $\mathcal{L}_{A}$, and $\mathcal{L}_{A}$ is associated with a penalty HP $\eta$; and 2) a structure autoencoder to reconstruct the structure, where the encoder is based on GAT \citep{velivckovic2017graph} and the decoder is a dot product operation followed by a \textit{sigmoid} function. This leads to the loss term $\mathcal{L}_{S}$, and $\mathcal{L}_{S}$ is associated with a penalty HP $\theta$.

Their overall optimization objective is then defined as $\mathcal{L} = \alpha\mathcal{L}_{S}+(1-\alpha)\mathcal{L}_{A},$ where $\alpha \in (0,1)$ balances the two objectives.

\textbf{HPs Sensitivity \& Tuning.} The paper finds that the AUC usually increases first and then decreases with the increase of $\alpha$. However, the specific value of $\alpha$ on each dataset is selected using label information. The HPs $(\alpha, \eta, \theta)$ are heuristically set as $(0.7,5,40)$, $(0.9,8,90)$, $(0.7,8,10)$ on BlogCatalog, Flickr, and ACM respectively.

\subsubsection{Revisiting SL-GAD} SL-GAD \citep{zheng2021generative} is an unsupervised GAD method that combines both contrastive and generative objectives. 

\textbf{Contrastive Framework---Data Augmentation Module.}
The method uses a single graph augmentation operation, namely Random Ego-Nets generation with a fixed size $K$. Specifically, taking the target node as the center, RWR \citep{tong2006fast} is used to generate two different subgraphs as ego-nets with a fixed size $K$, where $K$ controls the radius of the surrounding contexts. This gives one critical HP for graph augmentation, namely $K$. 

\textbf{Contrastive Framework---Contrast Learning Module.} The Multi-View Contrastive Learning module compares the similarity between a node embedding and the embedding of sampled sub-graphs in augmented views (namely node-subgraph contrast), leading to loss terms $\mathcal{L}_{con,1}$ and $\mathcal{L}_{con,2}$. Combining those leads to contrastive objective $\mathcal{L}_{con} = \frac{1}{2}(\mathcal{L}_{con,1}+\mathcal{L}_{con,2})$.

\textbf{Generative Framework.} The Generative Attribute Regression module reconstructs node attributes, with the aim to achieve node-level discrimination. Specifically, they minimize the Mean Square Error between the target node's original and reconstructed attributes in augmented views, leading to loss terms $\mathcal{L}_{gen,1}$ and $\mathcal{L}_{gen,2}$. Combining those with equal weights leads to generative objective $\mathcal{L}_{gen} = \frac{1}{2}(\mathcal{L}_{gen,1}+\mathcal{L}_{gen,2})$.

The overall optimization objective is then defined as $\mathcal{L} = \alpha\mathcal{L}_{con}+\beta\mathcal{L}_{gen},$
where $\alpha,\beta \in (0,1]$ are trade-off HPs to balance the importance of the two SSL objectives.

\textbf{HPs Sensitivity \& Tuning.} The authors conducted a sensitive analysis and found that: 1) the performance first increases and then decreases with the increase of $K$. For efficiency considerations, they heuristically set the sampled subgraph size $K=4$ for all datasets; 2) they heuristically fix $\alpha = 1$ for all datasets as they found that this achieves good performance on most datasets (with the help of label information); and 3) the selection of $\beta$ is highly dependent on the specific dataset. Hence, they ``fine-tune" the value of $\beta$ for each dataset via selecting $\beta$ from $\{0.2, 0.4, 0.6, 0.8, 1.0\}$ using labels.

\subsubsection{Other SSL-based GAD methods} Due to space constraints, the analyses of other SSL-based GAD methods, including CoLA \citep{liu2021anomaly} , GRADATE \citep{duan2023graph}, Sub-CR \citep{zhang2022reconstruction}, CONAD \citep{xu2022contrastive}, DOMINANT \citep{ding2019deep}, GUIDE \citep{yuan2021higher}, and GAAN \citep{chen2020generative}, are given in Appendix~\ref{appendix:pitfalls}. These methods are all representatives of recent advancements in using SSL to conduct unsupervised graph anomaly detection, and have yielded outstanding detection performance. Likewise, however, these methods also exhibit pitfalls with regard to hyperparameter tuning, similar to those of ANEMONE \citep{jin2021anemone}, AnomalyDAE \citep{fan2020anomalydae}, and SL-GAD \citep{zheng2021generative}.

\begin{table*}[!htbp]
\centering
    \caption{Performance variation, quantified as $\frac{\max(\text{AUC})-\min(\text{AUC})}{\max(\text{AUC})}$ across different hyperparameter settings on ten benchmark datasets. Results are averaged over five independent runs, each initialized with a unique random seed. `OOM' indicates out-of-memory errors, while `OOR' signifies that runtime exceeded the 7-day limit for a single trial. {\color{black} Cells marked as `UNF' denote persistent underfitting of algorithms, even after reaching the maximum allowed training epochs (e.g., loss values change by less than $10^{-2}$ after 400 epochs). `NAN' indicates execution errors caused by excessive NaN values; these cases are excluded from further analysis. Refer to Section~\ref{Sec:Experiments} for details on the experimental setup.}
}
\resizebox{\textwidth}{!}{
    \begin{tabular}{p{2cm}p{0.8cm}p{1cm}p{1cm}p{0.8cm}p{0.8cm}p{1.7cm}p{1cm}p{1.2cm}p{1.2cm}p{1.2cm}|p{1cm}}
        \toprule
         & Cora & CiteSeer & PubMed & ACM & Flickr & BlogCatalog & Amazon & Facebook & Reddit & YelpChi &\textbf{Average} \\
        \midrule
        CoLA & 1.1\% & 1.7\% & 1.6\% & 4.2\% & 3.3\% & 4.7\%  &18.0\% & 3.4\% & 2.9\% & 31.9\% &7.3\% \\
        ANEMONE & 8.9\% & 6.6\% & 6.3\% & 11.3\% & 16.9\% & 16.8\% & 32.7\% & 23.9\% & 8.9\% & 37.7\% & 17.0\%  \\
        GRADATE & 6.9\% & 14.1\% & \cellcolor{black!25} OOM & \cellcolor{black!25} OOM & \cellcolor{black!25} OOR & \cellcolor{black!25} OOR & 5.9\% & 22.9\% & \cellcolor{black!25} OOR & \cellcolor{black!25} OOR &12.5\% \\
        SL-GAD & 17.4\% & 16.2\% & 19.5\% & 17.7\% & 16.3\% & 23.4\% & 25.4\% & 47.8\% & 21.8\% & \cellcolor{black!25}OOR & 22.8\% \\
        Sub-CR & 15.1\% & 8.3\% & \cellcolor{black!25} OOM & \cellcolor{black!25} OOM & 9.8\% & 6.3\% & 28.6\% & 20.3\% & \cellcolor{black!25} OOM & \cellcolor{black!25} OOM & 14.7\% \\
        CONAD & 5.8\% & 7.0\% & 2.3\% & \cellcolor{black!25}UNF & \cellcolor{black!25} OOM & \cellcolor{black!25} OOM & 17.3\% &27.3\% & 9.7\% & 40.7\% & 15.7\% \\
        DOMINANT & 5.1\% & 6.0\% & 1.9\% & \cellcolor{black!25}UNF & \cellcolor{black!25}UNF & \cellcolor{black!25}UNF &12.4\% &19.1\% &8.4\% &34.5\% & 12.5\% \\
        A-DAE & 19.1\% & 25.3\% & 23.8\% & 20.8\% & 23.6\% & 14.3\% & 48.1\% & 64.9\% & \cellcolor{black!25} NAN & \cellcolor{black!25} NAN & 30.0\% \\
        GUIDE & 4.8\% & 4.8\% & 1.8\% & \cellcolor{black!25}UNF & 8.5\% & \cellcolor{black!25}UNF & 11.5\% & 18.6\% & 8.0\% & 34.4\% & 11.6\% \\
        GAAN & 28.1\% & 30.0\% & 30.3\% & 25.6\% & 10.1\% & 7.2\% & 13.1\% & 72.6\% & 11.9\% & 11.5\% & 24.0\% \\
        \bottomrule
    \end{tabular}
    }
    \label{tab:performance_variation}
\end{table*}
\subsection{Sensitivity Analysis}
\label{subsec:SystematicEvaluations}
After revisiting recent SSL-based unsupervised GAD methods, we now empirically investigate their sensitivity to SSL-related HPs in a systematic way. More concretely, we report their performance variations in terms of ROC-AUC values under different hyperparameter configurations (see Section~\ref{Sec:Experiments} for experiment settings). 

As shown in Figure~\ref{Fig:ROC_Var}, for a typical run with different hyperparameter configurations, the performance of ANEMONE \citep{jin2021anemone} can vary strongly on each of the {\color{black}ten} datasets. Other SSL-based GAD algorithms exhibit similar behavior; extensive results and analysis are deferred to Appendix~\ref{appendix:FurtherAnalysis} for space reasons.

For an in-depth yet compact analysis, Table~\ref{tab:performance_variation} presents average results over five independent runs when varying SSL-related hyperparameter values. {\color{black}Specifically, CoLA \citep{liu2021anomaly}, GUIDE \citep{yuan2021higher}, DOMINANT \citep{ding2019deep}, GRADATE \citep{duan2023graph}, and Sub-CR \citep{zhang2022reconstruction}  demonstrate moderate performance variations (namely between $7.3\%$ and $14.7\%$ on average). Meanwhile, CONAD \citep{xu2022contrastive}, ANEMONE \citep{jin2021anemone}, SL-GAD \citep{zheng2021generative}, GAAN \citep{chen2020generative}, and AnomalyDAE \citep{fan2020anomalydae} suffer from large performance variations (namely ranging from $15.7\%$ to $30.0\%$ on average)}. From Subsection \ref{subsec:pitfalls} and Appendix~\ref{appendix:pitfalls}, we see that the results reported in existing papers are often obtained by manually tuned HPs  (in a post-hoc way with label information), thereby leading to strongly overestimated performance for real-world applications where labels are not accessible. To mitigate this severe issue, we propose \textit{AutoGAD}, a method for automating hyperparameter selection in SSL for GAD and achieving truly unsupervised graph anomaly detection. Importantly, \textit{AutoGAD} does not need any ground-truth labels.

\section{AutoGAD: Using Internal Evaluation to Automate SSL for GAD}
Our proposed approach, called AutoGAD, consists of two parts: 1) an unsupervised performance metric, and 2) an effective search method. Importantly, and as mentioned before, the chosen performance metric---denoted $Metric[\cdot]$ in Equation~\ref{Equ:HPO}---should not use any ground-truth label information, simply because this is not available in a truly unsupervised setting. We therefore propose to utilize an internal evaluation strategy, which will be elucidated later. Next, given the impracticality of evaluating an infinite number of configurations for continuous hyperparameter domains, another challenge is the efficient exploration of the search space. Section \ref{subsec:Search} describes a straightforward approach using discretization and grid search that works well in practice, as shown in the next section.

\subsection{Internal Evaluation Strategy}\label{subsec:internal}


The intuition behind the internal evaluation strategy that we use is to measure the similarity of anomaly scores within the same predicted anomaly class and the dissimilarity between anomaly scores across different predicted classes (i.e., `anomaly' or `no anomaly'). As we will prove later, optimizing the resulting measure is equivalent to simultaneously minimizing the false positive rate and the false negative rate. In this way, we aim to evaluate and optimize the performance of the anomaly detector under different SSL configurations \emph{without having to rely on any ground-truth labels}. 

\subsubsection{Contrast Score Margin}
The metric that we use is Contrast Score Margin \citep{xu2019automatic}, which was introduced before but not for graph anomaly detection, and is defined as
\begin{equation}
    T(f) = \frac{\hat{\mu}_{\mathbf{O}} - \hat{\mu}_{\mathbf{I}}}{\sqrt{\frac{1}{k}(\hat{\delta}^{2}_{\mathbf{O}} + \hat{\delta}^{2}_{\mathbf{I}})}},
\label{Equ:CSM}
\end{equation}
where $\hat{\mu}_{\mathbf{O}}$ and $\hat{\delta}^{2}_{\mathbf{O}}$ denote the average and variance of the anomaly scores of the $k$ predicted anomalous objects ($\hat{\mathbf{O}}$), respectively. Moreover, $\hat{\mu}_{\mathbf{I}}$ and $\hat{\delta}^{2}_{\mathbf{I}}$ represent the average and variance of the anomaly scores of the $k$ predicted normal objects ($\hat{\mathbf{I}}$) with the highest scores, respectively. Intuitively, the metric focuses on the $k$ predicted normal objects that are most similar to the $k$ predicted anomalous objects, and aims to measure the margin of the anomaly scores between them. It only takes linear time with respect to $n$ to compute.

\subsubsection{Analysis}
We now analyze why the internal evaluation metric Contrast Score Margin should work for our purposes.

\begin{theorem}[Minimizing False Positives and Negatives]For an anomaly detector $f(\cdot)$ on dataset $\mathbf{X}$, assume the anomaly scores of the top $k$ true anomalies $(\mathbf{O})$ have the expected value $\mu_{\mathbf{O}}$ and variance $\delta_{\mathbf{O}}^{2}$, and the anomaly scores of the top $k$ true normal objects with the highest anomaly scores $(\mathbf{I})$ have the expected value $\mu_{\mathbf{I}}$ and variance $\delta_{\mathbf{I}}^{2}$, then maximizing T is equal to simultaneously minimizing the false positive rate and the false negative rate . 
\end{theorem}
\begin{proof}
According to \textit{Cantelli's inequality}, which makes no assumptions on specific probability distributions, on the one hand, for $\mathbf{x} \in \mathbf{O}$  we have $P(f(\mathbf{x}) \leq \mu_{\mathbf{O}} - \alpha) \leq \frac{\delta_{\mathbf{O}}^{2}}{\delta_{\mathbf{O}}^{2}+\alpha^{2}}$, where $\alpha \geq 0$ is a small constant chosen based on a desired bound on the false negative. By replacing $\alpha = a\delta_{\mathbf{O}}$, we have $P(f(\mathbf{x}) \leq \mu_{\mathbf{O}} - a\delta_{\mathbf{O}}) \leq \frac{1}{1+a^{2}}$, which is the False Negative Bound. In other words, $f(x)$ has a maximum probability of $\frac{1}{1+a^{2}}$ to be less than $\mu_{\mathbf{O}} - a\delta_{\mathbf{O}}$.

On the other hand, for $\mathbf{y} \in \mathbf{I}$ we 
have $P(f(\mathbf{y}) \geq \mu_{\mathbf{I}} + \beta) \leq \frac{\delta_{\mathbf{I}}^{2}}{\delta_{\mathbf{I}}^{2}+\beta^{2}}$, where $\beta \geq 0$ is a small constant chosen based on a desired bound on the false positive. By replacing $\beta = b\delta_{\mathbf{I}}$, we have $P(f(\mathbf{y}) \geq \mu_{\mathbf{I}} + b\delta_{\mathbf{I}}) \leq \frac{1}{1+b^{2}}$, which is the False Positive Bound. In other words, $f(y)$ has a maximum probability of $\frac{1}{1+b^{2}}$ to be larger than $\mu_{\mathbf{I}} + b\delta_{\mathbf{I}}$.

Furthermore, $(\mu_{\mathbf{O}} - a\delta_{\mathbf{O}}) - (\mu_{\mathbf{I}} + b\delta_{\mathbf{I}})$ = $(\mu_{\mathbf{O}} - \mu_{\mathbf{I}}) - (b\delta_{\mathbf{I}} + a\delta_{\mathbf{O}}) $. Hence, to ensure a small false positive rate and a small false negative rate, we want $\mu_{\mathbf{O}} - \mu_{\mathbf{I}}$ to be as large as possible while $b\delta_{\mathbf{O}} + a\delta_{\mathbf{I}}$ as small as possible. In fact, this is equivalent to optimize the Contrast Score Margin, i.e.,
\begin{equation*}
    T(f) = \frac{\mu_{\mathbf{O}} - \mu_{\mathbf{I}}}{\sqrt{\frac{1}{k}(\delta^{2}_{\mathbf{O}} + \delta^{2}_{\mathbf{I}})}}
\end{equation*}

Note that if an anomaly detector $f(\cdot)$ produces a perfect anomaly detection result, i.e., for any $\mathbf{x} \in \mathbf{O}$ and any $ \mathbf{y} \in \mathbf{X} \setminus \mathbf{O}$, we have $f(\mathbf{x}) > f(\mathbf{y})$, then we will obtain  $\mu_{\mathbf{O}} - \mu_{\mathbf{I}} > 0$. In another extreme, if $f(\cdot)$ produces a poor anomaly detection result, i.e., for all $\mathbf{x} \in \mathbf{O}$ and any $ \mathbf{y} \in \mathbf{X} \setminus \mathbf{O}$, we have $f(\mathbf{x}) < f(\mathbf{y})$, then we will obtain  $\mu_{\mathbf{O}} - \mu_{\mathbf{I}} < 0$. Meanwhile, if an anomaly detector $f(\cdot)$ produces a random result, i.e., for some $\mathbf{x} \in \mathbf{O}$ and any $ \mathbf{y} \in \mathbf{X} \setminus \mathbf{O}$, we have $f(\mathbf{x}) < f(\mathbf{y})$, then we may obtain $\mu_{\mathbf{O}} - \mu_{\mathbf{I}} < 0$ or $\mu_{\mathbf{O}} - \mu_{\mathbf{I}} \approx 0$.
\end{proof}

\subsubsection{Improvements and Remarks}
In practice we observed that Equation~\ref{Equ:CSM} is not always stable. Possible reasons are that 1) the proportion of anomalies is usually very small (namely less than 5\% in most datasets); and 2) the exact number of anomalies is generally not known (even for a dataset with injected anomalies, there may exist some natural samples that exhibit similar behaviors as anomalies). Therefore, we propose to modify Equation~\ref{Equ:CSM} as follows:
\begin{equation}
\label{Equ:ImprovedCSM}
    T(f) = \frac{\hat{\mu}_{\mathbf{O}} - \Tilde{\mu}_{\mathbf{I}}}{\sqrt{\hat{\delta}^{2}_{\mathbf{O}} + \Tilde{\delta}^{2}_{\mathbf{I}}}},
\end{equation}
where {$\hat{\mu}_{\mathbf{O}}$} and $\hat{\delta}^{2}_{\mathbf{O}}$ denote the average and variance of the anomaly scores of the $k$ predicted anomalous objects, respectively. Importantly, $\Tilde{\mu}_{\mathbf{I}}$ and $\Tilde{\delta}^{2}_{\mathbf{I}}$ represent the average and variance of the anomaly scores of the remaining $n-k$ objects, respectively. This change should lead to more stable performance compared to using anomaly scores of the top-$k$ predicted normal objects in Equation~\ref{Equ:CSM}. This is because the true labels are not accessible, and thus we utilize the pseudo-labels to identify the  top-$k$  anomalous and the  top-$k$  normal objects. However, the pseudo-labels of the  top-$k$ ``pseudo-normal" objects may not be reliable due to the two facts stated above.

Moreover, to ensure the effectiveness of this internal evaluation strategy, we have to make sure that: 1) we use the same algorithm with different hyperparameter configurations; and 2) the scales of the loss values are approximately the same when combining multiple loss functions in the same algorithm. In other words, we should not directly use the strategy to select among different heterogeneous anomaly detection algorithms {(\color{black}please refer to Appendix~\ref{appendix:AutoGAD4ADSelection} for empirical evidence of this)}.

\subsection{Discretization and Grid Search}
\label{subsec:Search}
\begin{algorithm}
\caption{ {\color{black}Grid Search for Anomaly Detector Hyperparameter Optimization}}
\label{alg:grid_search}
\begin{algorithmic}[1]
\Require Graph anomaly detection algorithm $f(\cdot)$, graph $\mathcal{G}$, hyperparameter domains $\boldsymbol{\Lambda} = \{\boldsymbol{\Lambda}^{(1)}, \ldots, \boldsymbol{\Lambda}^{(L)}\}$, internal evaluation function $T(\cdot)$
\Ensure Best hyperparameter configuration $\boldsymbol{\lambda}_{\text{best}}$
\State Discretize each continuous domain $\boldsymbol{\Lambda}^{(l)}$ into a finite set if necessary
\State Generate hyperparameter search set $\boldsymbol{\lambda}_{\text{search}} = \{\boldsymbol{\lambda}_{1}, \ldots, \boldsymbol{\lambda}_{M}\}$ where $M = \prod_{l=1}^{L} |\boldsymbol{\Lambda}^{(l)}|$
\State Initialize best score $t_{\text{best}} \gets -\infty$ and best configuration $\boldsymbol{\lambda}_{\text{best}} \gets \emptyset$
\For{each $\boldsymbol{\lambda}_{m} \in \boldsymbol{\lambda}_{\text{search}}$}
    \State Compute anomaly scores $\boldsymbol{s}_{m}(\mathcal{G}) = f(\boldsymbol{\lambda}_{m}; \mathcal{G})$
    \State Compute evaluation score $t_{m}(\mathcal{G}) = T(\boldsymbol{s}_{m}(\mathcal{G}))$
    \If{$t_{m}(\mathcal{G}) > t_{\text{best}}$}
        \State Update $t_{\text{best}} \gets t_{m}(\mathcal{G})$
        \State Update $\boldsymbol{\lambda}_{\text{best}} \gets \boldsymbol{\lambda}_{m}$
    \EndIf
\EndFor
\State \Return $\boldsymbol{\lambda}_{\text{best}}$
\end{algorithmic}
\end{algorithm}

{\color{black} To find the optimal hyperparameter configuration, we first perform discretization of the continuous search space and then conduct grid search. The corresponding pseudo-code is provided in Algorithm~\ref{alg:grid_search}, with a detailed explanation presented below.

\textbf{Discretization of Continuous Search Space} (Lines 1--2).
To make the overall search process feasible, we discretize the hyperparameter space. Assume we are given a GAD algorithm $f(\cdot)$ with its set of hyperparameters $\boldsymbol{\lambda} \in \boldsymbol{\Lambda}$. Without loss of generality, we assume there are $L$ different hyperparameters and let $\boldsymbol{\lambda} = \{\lambda^{(1)}, \lambda^{(2)}, \ldots, \lambda^{(L)}\}$, where each $\lambda^{(l)} \in \boldsymbol{\Lambda^{(l)}}$ for $l = 1, 2, \ldots, L$. If a hyperparameter domain $\boldsymbol{\Lambda^{(l)}}$ is continuous, we discretize it into a finite set of values (with cardinality $\vert \boldsymbol{\Lambda^{(l)}} \vert$). This results in $M$ possible hyperparameter configurations, represented by the set $\boldsymbol{\lambda}_{\text{search}} = \{\boldsymbol{\lambda}_{1}, \ldots, \boldsymbol{\lambda}_{m},\ldots,\boldsymbol{\lambda}_{M}\}$, where $\boldsymbol{\lambda}_{m} = \{\lambda_{m}^{(1)},\lambda_{m}^{(2)},\cdots,\lambda_{m}^{(L)}\}$ and $M = \prod_{l=1}^{L} \vert \boldsymbol{\Lambda^{(l)}} \vert$.

\textbf{Grid Search} (Lines 3--11).
Once the hyperparameter search space is discretized, we apply grid search to evaluate each configuration. For each hyperparameter configuration $\boldsymbol{\lambda}_{m} \in \boldsymbol{\lambda}_{\text{search}}$, we run the GAD algorithm $f(\boldsymbol{\lambda}_{m})$ on the given graph $\mathcal{G}$ to produce a vector of anomaly scores $\boldsymbol{s}_{m}(\mathcal{G}) = f(\boldsymbol{\lambda}_{m}; \mathcal{G})$. These scores are evaluated using an internal unsupervised performance metric $T(\cdot)$ (with Equation~\ref{Equ:ImprovedCSM}) to yield a final score $t_{m}(\mathcal{G}) = T(\boldsymbol{s}_{m}(\mathcal{G}))$. The configuration that maximizes $T(\cdot)$ is selected as the optimal values of hyperparameters.

Note that more advanced strategies than grid search, such as SMBO-based optimization \citep{jones1998efficient}, could be employed (see Appendix \ref{appendix:SMBO} for an example). However, these methods often introduce additional hyperparameters (whose tuning may be non-trivial), which contradicts our goal of automated anomaly detection.
}

\section{Experiments}
\label{Sec:Experiments}

We aim to answer the following research questions (RQ):
\begin{itemize}
    \item[\textbf{RQ1}] How sensitive are existing SSL-based GAD methods to the values of their hyperparameters?
    \item[\textbf{RQ2}] How effective is \textit{AutoGAD} in tuning SSL-related hyperparameter values for these methods?
\end{itemize}
We describe the experiment settings, including the datasets, baselines, evaluation metrics, and software and hardware used, which is followed by the experiment results and their interpretation.

\subsection{Datasets}

We use three popular citation networks, namely Cora, Citeseer, and Pubmed \citep{sen2008collective} with injected anomalies, one social network Flickr \citep{zeng2019graphsaint} (less homophily) with injected anomalies, ACM \citep{tang2008arnetminer} as well as BlogCataLog \citep{zeng2019graphsaint} with injected anomalies. Particularly, we follow the methods used by ANEMONE \citep{jin2021anemone} and CoLA \citep{liu2021anomaly} to inject structural and contextual anomalies. Note that \cite{liu2022bond} have slightly modified this injection procedure. {\color{black} Following \citep{qiao2024truncated}, we also consider four commonly-used graph datasets with real anomalies: Amazon \citep{sanchez2013statistical}, Facebook \citep{leskovec2012learning}, Reddit \citep{kumar2019predicting}, and YelpChi \citep{rayana2015collective}.} The resulting datasets are summarized in Table~\ref{tab:data_summary}.

\begin{table}[!tbp]
	\caption{ {\color{black} Summary of datasets: anomalies in Cora, CiteSeer, PubWeb, ACM, BlogCatalog, and Flickr are synthetically injected following established methods \citep{jin2021anemone,liu2021anomaly}, while Amazon, Facebook, Reddit, and YelpChi contain real-world anomalies.}}
	\centering
	\begin{tabular}{ccccc}
		\toprule
		Dataset & \#Nodes & \#Edges & \#Attributes &\#Anomalies \\
		\midrule
            Cora \citep{sen2008collective} & 2708 & 11060 & 1433 & 138(5.1\%)\\
            CiteSeer \cite{sen2008collective} & 3327 & 4732 & 3703 &150(4.5\%)\\
            PubMed \citep{sen2008collective} & 19717 & 44338 & 500 &150(2.5\%)\\
            ACM \citep{tang2008arnetminer} & 16484 & 71980 & 8337 &600(3.6\%)\\
            BlogCataLog \citep{zeng2019graphsaint} & 5196 & 171743 & 8189 &300(5.8\%)\\
            Flickr \citep{zeng2019graphsaint} & 7575 & 239738 & 12407 & 450(5.9\%)\\
            Amazon \citep{sanchez2013statistical} & 10244 & 175608 & 25 & 693(6.7\%)\\
            Facebook \citep{leskovec2012learning} & 1081 & 55104	& 576 &	27(2.5\%)\\
            Reddit \citep{kumar2019predicting} &10984 & 168016	& 64	&366(3.3\%)\\
            YelpChi \citep{rayana2015collective} &24741	&49315	&32	&1217(4.9\%)\\
        \hline
	\end{tabular}
	\label{tab:data_summary}
\end{table}

\subsection{Baselines}
We study the performance of the following SSL-based graph anomaly detection methods:
\begin{itemize}
    \item Generative methods: DOMINANT \citep{ding2019deep}, AnomalyDAE \citep{fan2020anomalydae}, GUIDE \citep{yuan2021higher}, GAAN \citep{chen2020generative};
    \item Contrastive methods (and some also generative): CoLA \citep{liu2021anomaly}, ANEMONE \citep{jin2021anemone}, GRADATE \citep{duan2023graph}, SL-GAD \citep{zheng2021generative}, Sub-CR \citep{zhang2022reconstruction}, CONAD \citep{xu2022contrastive}.
\end{itemize}
Particularly, the SSL-related HPs for each GAD algorithm and their discretized search spaces are given in Table~\ref{tab:HP_search} in the Appendix.
These GAD methods are further summarized in Table~\ref{tab:algo_summary} in the Appendix.

\subsection{Evaluation Metrics}
To evaluate the effectiveness of various GAD algorithms, we utilize the ROC-AUC metric \citep{hanley1982meaning} ($AUC$ for short hereinafter), where a value approaching 1 denotes the best possible performance.

Moreover, to quantify the performance variation of an individual GAD method under different SSL-related HP configurations, we define the following \textit{performance variation} metric:
\begin{equation}
    \frac{\max(AUC)-\min(AUC)}{\max(AUC)},
\end{equation}
where $\max(AUC)$ and $\min(AUC)$ represent the maximum and minimum of achieved $AUC$ values for the evaluated GAD algorithm with different configurations, respectively. Hence, the smaller this value is, the less sensitive the algorithm is to SSL-related HPs. 

Further, we define the \textit{performance gain over minimal AUC} as
\begin{equation}
    \frac{\text{CSM}(AUC)-\min(AUC)}{\min(AUC)},
\end{equation}
where \text{CSM}($AUC$) indicates the $AUC$ value obtained for the evaluated GAD algorithm when configured with the HPs selected using the Contrast Score Margin. This metric can quantify the effectiveness of our strategy relative to the worst case hyperparameter setting. Next, we define \textit{performance gain over median AUC} 
as
\begin{equation}
    \frac{\text{CSM}(AUC)-\text{median}(AUC)}{\text{median}(AUC)},
\end{equation}
where $\text{median}(AUC)$ represents the median of the obtained $AUC$ values for the GAD algorithm with different configurations. Thus, if the value of this metric is positive, the GAD algorithm configured with our selected HPs can at least outperform its counterparts configured with $50\%$ of the other sampled hyperparameter values.

Furthermore, we define \textit{performance gain over maximal AUC} 
as
\begin{equation}
    \frac{\text{CSM}(AUC)-\text{max}(AUC)}{\text{max}(AUC)},
\end{equation}
where $\text{max}(AUC)$ represents the maximum of the obtained $AUC$ values for the GAD algorithm with different configurations. Thus, if the value of this metric is close to zero, the GAD algorithm configured with our selected HPs can approximately achieve the best possible performance.

\subsection{Software and Hardware}
All algorithms are implemented in Python 3.8 (using PyTorch \citep{paszke2019pytorch}
and PyTorch Geometric \citep{fey2019fast} libraries when applicable) and ran on workstations
equipped with AMD EPYC7453 CPUs (with 64GB RAM) and/or Nvidia RTX4090 GPUs (with 24.0 GB video memory). All code and datasets are available on GitHub\footnote{\url{https://github.com/ZhongLIFR/AutoGAD2024}}.

\subsection{Results and Analysis}

\begin{table*}[!h]
    \caption{Performance gain over minimal AUC defined as $\frac{\text{CSM}(AUC)-\min(AUC)}{\min(AUC)}$. Results are averaged on five independent runs. CSM is contrast score margin defined in Equation~\ref{Equ:ImprovedCSM}, {\color{black}while OOM, OOR, UNF, and NAN convey the same meanings as} in Table~\ref{tab:performance_variation}.}
    \centering
\resizebox{\textwidth}{!}{
    \begin{tabular}{p{2cm}p{1.1cm}p{1.1cm}p{1.1cm}p{1.1cm}p{1.1cm}p{1.7cm}p{1.1cm}p{1.1cm}p{1.1cm}p{1.1cm}|p{1cm}}
        \toprule
         & Cora & CiteSeer & PubMed & ACM & Flickr & BlogCatalog & Amazon & Facebook & Reddit & YelpChi &\textbf{Average} \\
        \midrule
        CoLA & 0.5\% & 1.2\% & 1.6\% & 2.8\% & 2.2\% & 4.7\% & 22\% & 1.8\% & 1.5\% & 2.5\% & 4.1\% \\
        ANEMONE & 8.6\% & 5.9\% & 6.6\% & 6.8\% & 15.8\% & 19.7\% & 44.7\% & 28.1\% & 4.0\% & 11.9\% & 15.2\%\\
        GRADATE & 4.0\% & 14.3\% & \cellcolor{black!25} OOM & \cellcolor{black!25} OOM & \cellcolor{black!25} OOR & \cellcolor{black!25} OOR & 4.3\% & 29.7\% & \cellcolor{black!25} OOR & \cellcolor{black!25} OOR & 13.1\% \\
        SL-GAD & 21.2\% & 19.1\% & 23.7\% & 21.3\% & 18.2\% & 30.4\% & 13.0\% & 16.3\% & 15.8\% & \cellcolor{black!25}OOR & 19.9\%\\
        Sub-CR & 16.2\% & 4.3\% & \cellcolor{black!25} OOM & \cellcolor{black!25} OOM & 4.3\% & 2.4\% &19.2\% & 25.3\%&\cellcolor{black!25} OOM &\cellcolor{black!25} OOM & 12.0\%\\
        CONAD & 5.4\% & 2.3\% & 2.1\% & \cellcolor{black!25}UNF & \cellcolor{black!25} OOM & \cellcolor{black!25} OOM &6.5\% &18.3\% &2.4\% &24.3\% &8.8\% \\
        DOMINANT & 5.2\% & 1.3\% & 1.8\% & \cellcolor{black!25}UNF & \cellcolor{black!25}UNF & \cellcolor{black!25}UNF &13.7\% &14.9\% &0.8\% &28.0\% &9.4\% \\
        A-DAE & 11.3\% & 4.6\% & 11.9\% & 5.6\% & 30.8\% & 32.2\% & 67.3\% &114.6\% & \cellcolor{black!25} NAN & \cellcolor{black!25} NAN &34.8\% \\
        GUIDE & 5.0\% & 1.2\% & 1.8\% & \cellcolor{black!25}UNF & 0.1\% & \cellcolor{black!25}UNF & 9.4\% & 14.3\% & 3.8\% & 28.8\% &8.1\% \\
        GAAN & 7.7\% & 34.1\% & 43.6\% & 5.6\% & 1.3\% & 6.6\% & 12.4\% & 77.9\% & 0.4\% & 14.5\% &20.4\% \\
        \bottomrule
    \end{tabular}
}
    \label{tab:results_gains_min_auc}
\end{table*}

\begin{table*}[!h]
    \caption{Performance gain over median AUC defined as $\frac{\text{CSM}(AUC)-\text{median}(AUC)}{\text{median}(AUC)}$. Results are averaged on five independent runs. CSM is contrast score margin defined in Equation~\ref{Equ:ImprovedCSM}, {\color{black}while OOM, OOR, UNF, and NAN convey the same meanings as in Table~\ref{tab:performance_variation}. For enhanced readability, cells are color-coded based on their values, as specified in the legend.}}
    \centering
\resizebox{0.6\textwidth}{!}{%
\begin{tabular}{|c|c|c|c|c|}
\hline
\cellcolor{orange!75}\textbf{Dark Orange} & \cellcolor{orange!25}\textbf{Light Orange} & 
\cellcolor{green!25}\textbf{Light Green} & \cellcolor{green!75}\textbf{Dark Green} & \cellcolor{black!25}\textbf{Grey} \\
\hline
$(-\infty,-5.0\%]$ & $(-5.0\%, 0.0\%)$ & $[0.0\%, 5.0\%)$ & $[5.0\%,+\infty)$ & Excluded \\
\hline
\end{tabular}
}
\vspace{0.5cm}
\resizebox{\textwidth}{!}{
\begin{tabular}{p{2cm}p{1.1cm}p{1.1cm}p{1.1cm}p{1.1cm}p{1.1cm}p{1.7cm}p{1.1cm}p{1.1cm}p{1.1cm}p{1.1cm}|p{1cm}}
    \toprule
     & Cora & CiteSeer & PubMed & ACM & Flickr & BlogCatalog & Amazon & Facebook & Reddit & YelpChi &\textbf{Average} \\
    \midrule
    CoLA & \cellcolor{orange!25} -0.1\% & \cellcolor{green!25} 0.1\% & \cellcolor{green!25} 0.2\% & \cellcolor{orange!25} -0.7\% & \cellcolor{green!25} 0.6\% &  \cellcolor{green!25} 2.0\% & \cellcolor{green!75}15.4\% & \cellcolor{green!25} 0.1\% & \cellcolor{orange!25} -0.2\% &  \cellcolor{orange!25}-3.3\% & \cellcolor{green!25}1.4\% \\
    ANEMONE & \cellcolor{green!25} 0.3\% & \cellcolor{green!25} 1.0\% & \cellcolor{green!25} 1.5\% & \cellcolor{orange!25} -0.8\% & \cellcolor{green!25} 2.0\% & \cellcolor{green!75} 7.2\% & \cellcolor{green!75}26.4\% & \cellcolor{green!25} 3.5\% & \cellcolor{orange!25} -2.4\% & \cellcolor{green!25} 1.6\% & \cellcolor{green!25} 4.0\% \\
    GRADATE & \cellcolor{orange!25} -0.6\% & \cellcolor{green!25} 4.0\% & \cellcolor{black!25}OOM & \cellcolor{black!25}OOM & \cellcolor{black!25}OOR & \cellcolor{black!25}OOR & \cellcolor{green!25} 0.7\% & \cellcolor{green!75} 18.3\% & \cellcolor{black!25}OOR & \cellcolor{black!25}OOR &\cellcolor{green!75}5.6\% \\
    SL-GAD & \cellcolor{green!25} 3.3\% & \cellcolor{green!25} 3.7\% & \cellcolor{green!25} 4.8\% & \cellcolor{green!25} 4.3\% & \cellcolor{green!25} 2.8\% & \cellcolor{green!75} 5.0\% & \cellcolor{orange!25} -1.4\% & \cellcolor{orange!75} -31.6\% & \cellcolor{orange!25} -2.0\%   & \cellcolor{black!25}OOR &  \cellcolor{orange!25} -1.2\%\\
    Sub-CR & \cellcolor{orange!25} -1.6\% & \cellcolor{orange!25} -0.4\% & \cellcolor{black!25}OOM & \cellcolor{black!25}OOM & \cellcolor{orange!25} -3.2\% & \cellcolor{orange!25} -2.2\% &\cellcolor{green!25} 2.6\% &\cellcolor{green!75} 9.8\% &\cellcolor{black!25}OOM &\cellcolor{black!25}OOM & \cellcolor{green!25} 0.8\% \\
    CONAD & \cellcolor{green!25} 4.0\% & \cellcolor{green!25} 1.2\% & \cellcolor{green!25} 1.5\% & \cellcolor{black!25}UNF & \cellcolor{black!25}OOM & \cellcolor{black!25}OOM &\cellcolor{orange!25}-3.7\% &\cellcolor{green!25}2.5\% &\cellcolor{orange!25}-3.1\% &\cellcolor{green!25}1.5\% &\cellcolor{green!25}0.6\% \\
    DOMINANT & \cellcolor{green!25} 3.7\% & \cellcolor{green!25} 0.5\% & \cellcolor{green!25} 1.3\% & \cellcolor{black!25}UNF & \cellcolor{black!25}UNF & \cellcolor{black!25}UNF &\cellcolor{green!25}4.4\% &\cellcolor{orange!25}-1.8\% &\cellcolor{orange!25}-3.0\% &\cellcolor{green!25}4.8\% &\cellcolor{green!25}1.4\% \\
    A-DAE & \cellcolor{green!25} 2.9\% & \cellcolor{orange!25} -3.0\% & \cellcolor{orange!25} -2.9\% & \cellcolor{orange!25} -4.1\% & \cellcolor{green!25} 0.9\% & \cellcolor{orange!25} -3.4\% &\cellcolor{green!25}2.6\% &\cellcolor{green!25}0.7\% & \cellcolor{black!25} NAN & \cellcolor{black!25} NAN &\cellcolor{orange!25}-0.8\% \\
    GUIDE & \cellcolor{green!25} 3.8\% & \cellcolor{green!25} 0.6\% & \cellcolor{green!25} 1.5\% & \cellcolor{black!25}UNF & \cellcolor{orange!25} -0.3\% & \cellcolor{black!25}UNF &\cellcolor{green!25}2.1\% &\cellcolor{orange!25}-2.3\% & \cellcolor{green!25} 0.2\% & \cellcolor{green!75} 5.3\% &\cellcolor{green!25}1.4\% \\
    GAAN & \cellcolor{green!25} 2.8\% & \cellcolor{green!75} 27.5\% & \cellcolor{green!75} 35.6\% & \cellcolor{green!25} 3.5\% & \cellcolor{green!25} 0.7\% & \cellcolor{green!25} 4.7\% & \cellcolor{orange!25} -2.4\% & \cellcolor{orange!75} -45.6\% & \cellcolor{orange!25} -1.3\% & \cellcolor{orange!25} -0.5\% &\cellcolor{green!25}2.5\% \\
    \bottomrule
\end{tabular}
}
    \label{tab:results_gains_median_auc}
\end{table*}

\begin{table*}[!h]
	\caption{Performance gain over maximal AUC defined as $\frac{\text{CSM}(AUC)-\text{max}(AUC)}{\text{max}(AUC)}$. Results are averaged on five independent runs. CSM is contrast score margin defined in Equation~\ref{Equ:ImprovedCSM}, while OOM, OOR, and NAN convey the same meanings as in Table~\ref{tab:performance_variation}. }
	\centering
\resizebox{\textwidth}{!}{
    \begin{tabular}{p{2cm}p{1.1cm}p{1.1cm}p{1.1cm}p{1.1cm}p{1.1cm}p{1.7cm}p{1.1cm}p{1.2cm}p{1.1cm}p{1.1cm}|p{1.1cm}}
        \toprule
         & Cora & CiteSeer & PubMed & ACM & Flickr & BlogCatalog & Amazon & Facebook & Reddit & YelpChi &\textbf{Average} \\
		\midrule
		CoLA & -0.6\% & -0.5\% & -0.1\% & -1.5\% & -1.3\% & -0.3\% &0\% & -1.7\% &-1.5\% & -30.2\%&-3.8\% \\
		ANEMONE & -1.1\% & -1.1\% & -0.2\% & -5.4\% & -3.9\% & -0.4\% & -2.6\% & -2.6\% & -5.3\% & -30.3\% & -5.3\% \\
		GRADATE & -3.2\% & -1.9\% & \cellcolor{black!25} OOM & \cellcolor{black!25} OOM & \cellcolor{black!25} OOR & \cellcolor{black!25} OOR & -1.8\% & 0.0\% & \cellcolor{black!25} OOR & \cellcolor{black!25} OOR &-1.7\% \\
		SL-GAD & -0.4\% & -0.3\% & -0.4\% & -0.4\% & -1.2\% & -0.2\% & -15.8\% & -39.4\% & -9.4\% & \cellcolor{black!25} OOR & -7.5\% \\
		Sub-CR & -3.8\% & -4.5\% & \cellcolor{black!25} OOM & \cellcolor{black!25} OOM & -5.9\% & -4.1\% & -15.3\%& -0.2\% & \cellcolor{black!25} OOM & \cellcolor{black!25} OOM & -5.6\% \\
		CONAD & -0.8\% & -4.9\% & -0.2\% & \cellcolor{black!25}UNF & \cellcolor{black!25} OOM & \cellcolor{black!25} OOM &-11.9\% &-14.0\% &-7.6\% &-26.3\% &-9.3\% \\
		DOMINANT & -0.2\% & -4.8\% & -0.1\% & \cellcolor{black!25}UNF & \cellcolor{black!25}UNF & \cellcolor{black!25}UNF &-0.6\% &-7.1\% &-7.8\% &-16.2\% & -5.3\% \\
		A-DAE & -10.0\% & -21.8\% & -14.6\% & -16.4\% & -0.1\% & -4.8\% &-18.5\% &-26.3\% & \cellcolor{black!25} NAN & \cellcolor{black!25} NAN & -14.1\% \\
		GUIDE & 0\% & -3.6\% & 0\% & \cellcolor{black!25}UNF & -8.4\% & \cellcolor{black!25}UNF &-3.1\% &-7.1 \% & -4.6\% &-15.4\% &-5.3\% \\
		GAAN & -22.6\% & -6.7\% & 0\% & -21.6\% & -8.9\% & -1.2\% &-2.6\% &-53.7\% & -11.6\% & -0.9\%&-13.0\% \\
		\bottomrule
	\end{tabular}
    }
	\label{tab:results_gains_max_auc}
\end{table*}

We answer the research questions as follows:

\subsubsection{\textbf{RQ1}: Sensitivity of SSL-based GAD methods to HPs}

The results are summarized in Table~\ref{tab:performance_variation} for five independent runs. Typical runs are depicted in Figure \ref{Fig:ROC_Var} and in Figures \ref{Fig:ROC_Var_CoLA}-\ref{Fig:ROC_Var_GAAN} in Appendix \ref{appendix:FurtherAnalysis}. We briefly analyzed the results in Subsection \ref{subsec:SystematicEvaluations}; more detailed analyses are given in Appendix \ref{appendix:FurtherAnalysis}. To recall, {\color{black}five} out of ten algorithms show moderate performance variations, while the remaining {\color{black}five} algorithms demonstrate large performance variations when the values of SSL-related HPs are varied. In other words, SSL-based GAD methods are (sometimes highly) sensitive to hyperparameter values.

\subsubsection{\textbf{RQ2}: Effectiveness of AutoGAD in tuning SSL-related HPs} The results are summarized in Tables~\ref{tab:results_gains_min_auc}, \ref{tab:results_gains_median_auc} and \ref{tab:results_gains_max_auc} for five independent runs, while Figure \ref{Fig:ROC_Var} and Figures \ref{Fig:ROC_Var_CoLA}-\ref{Fig:ROC_Var_GAAN} depict typical runs. We have the following main observations:
\begin{itemize}
    \item[1)] From Table~\ref{tab:results_gains_min_auc}, {\color{black}one can see that \textit{AutoGAD} can result in moderate \textit{performance gain over minimal AUC} (namely between $4.1\%$ and $13.1\%$ on average) for CoLA, GUIDE, CONAD, DOMINANT, Sub-CR, and GRADATE. Recall that five of these algorithms (including CoLA, GUIDE, DOMINANT, GRADATE, and Sub-CR) exhibit moderate performance variations, ranging from $7.3\%$ to $14.7\%$ on average. Moreover, \textit{AutoGAD} leads to large \textit{performance gain over minimal AUC} (namely between $15.2\%$ and $34.8\%$ on average) for the remaining four algorithms, which suffer from large performance variations (namely between $17.0\%$ and $30.0\%$ on average). Overall, \textit{AutoGAD} is substantially better than the worst case, i.e., when one happens to select the HP values that give the smallest $AUC$ value.}
    \item[2)] From  Table~\ref{tab:results_gains_median_auc}, {\color{black}one can see that \textit{AutoGAD} can result in positive \textit{performance gain over median AUC} in 8 out 10 algorithms (ranging from $0.6\%$ to $5.6\%$ on average), implying that the HP values selected by \textit{AutoGAD} are  better than at least $50\%$ of randomly selected HP values. Particularly, the \textit{performance gains over median AUC} for GRADATE \citep{duan2023graph}, ANEMONE \citep{jin2021anemone}, and GAAN \citep{chen2020generative} are $5.6\%, 4.0\%, $ and $2.5\%$ respectively, which shows that \textit{AutoGAD} is highly effective for these methods.}
    \item[3)] From Table~\ref{tab:results_gains_max_auc}, {\color{black}one can see that \textit{AutoGAD} can result in  \textit{performance gain over max AUC} larger than $-10\%$ in 8 out 10 algorithms, implying that the HP values selected by \textit{AutoGAD}  can achieve performances that are comparable to optimal performances. For instance, the \textit{performance gains over max AUC} for GRADATE and SL-GAD are $-1.7\%$ and $ -7.5\%$ respectively, which shows that \textit{AutoGAD} is highly effective for these methods while they show moderate or large performance variations ($12.5\%$ and $22.8\%$ respectively).}
    \item[4)] Following the above observations, {\color{black}we check the details in Figure~\ref{Fig:ROC_Var_SL-GAD} for SL-GAD, Figure~\ref{Fig:ROC_Var_GRADATE} for GRADATE, and
    Figure~\ref{Fig:ROC_Var} for ANEMONE. For SL-GAD and GRADATE, \textit{AutoGAD} often selects HP values better than $90\%$ of randomly selected HPs values on most datasets. For ANEMONE, the HP values selected by \textit{AutoGAD} often outperform $75\%$ of randomly selected HP values}.
\end{itemize}

\subsubsection{Sensitivity Analysis}

\begin{figure}
\centering
\includegraphics[width=10cm]{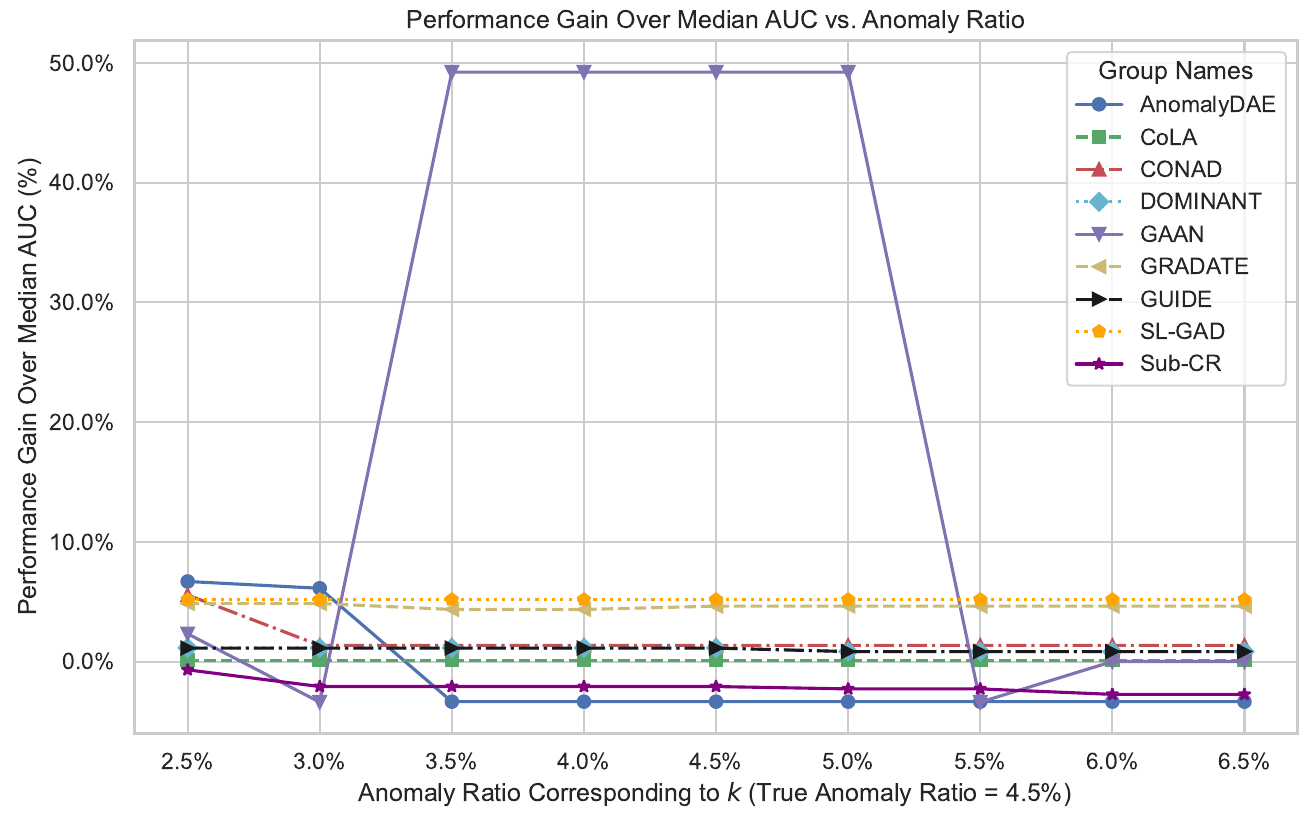}
\caption{Sensitivity analysis of $k$ (for our proposed AutoGAD) on dataset CiteSeer with all investigated SSL-based GAD algorithms. It can be seen that AutoGAD remains stable as long as  $k$ is not drastically distant from the actual anomaly ratio (namely $4.5\%$) for all SSL-based GAD algorithms.}
\label{Fig:SensitivityAnalysis}
\end{figure}

\textbf{Sensitivity to $k$.} The selection of the value of $k$ in our experiments acknowledges the varying anomaly ratios across different datasets, implying that  $k$ should ideally differ to reflect the unique characteristics of each dataset. We operated under the assumption that the anomaly ratio within a dataset is approximately known, a premise that aligns with real-world anomaly detection tasks where some prior knowledge about the frequency of anomalies is often available.

As shown in Figure~\ref{Fig:SensitivityAnalysis}, we conducted a sensitivity analysis on $k$ to assess the stability of AutoGAD against deviations from the true anomaly ratio. The findings from this analysis indicate that the effectiveness of AutoGAD remains stable as long as  $k$ is not drastically distant from the actual anomaly ratio, reinforcing the practical applicability of our approach even when exact anomaly proportions are not precisely determined.

\begin{figure}[h!]
    \centering
    \begin{minipage}[t]{0.32\linewidth}
        \centering
        \includegraphics[width=\linewidth]{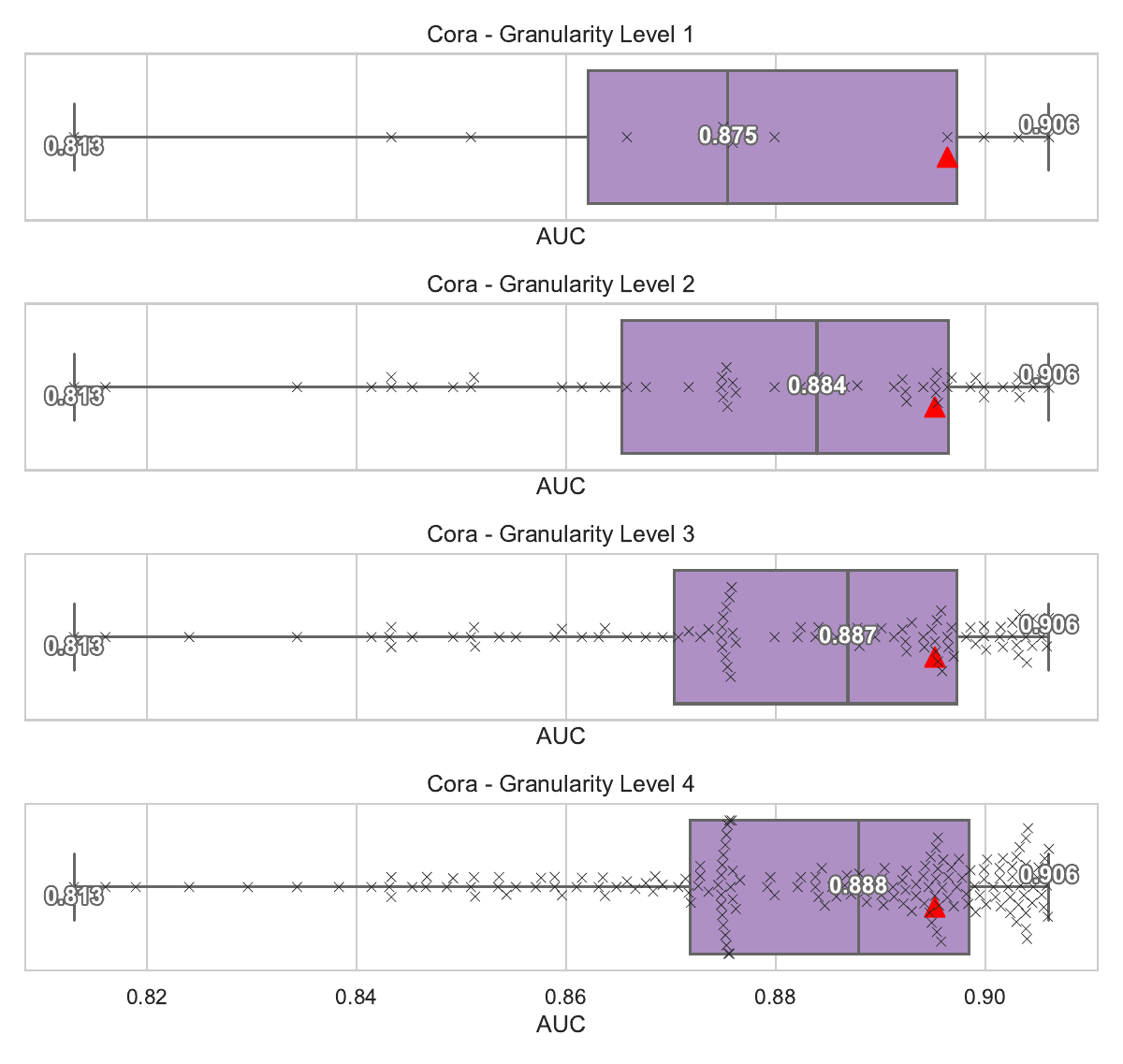}
        \label{fig:cora-gra-level}
    \end{minipage}
    \hfill
    \begin{minipage}[t]{0.32\linewidth}
        \centering
        \includegraphics[width=\linewidth]{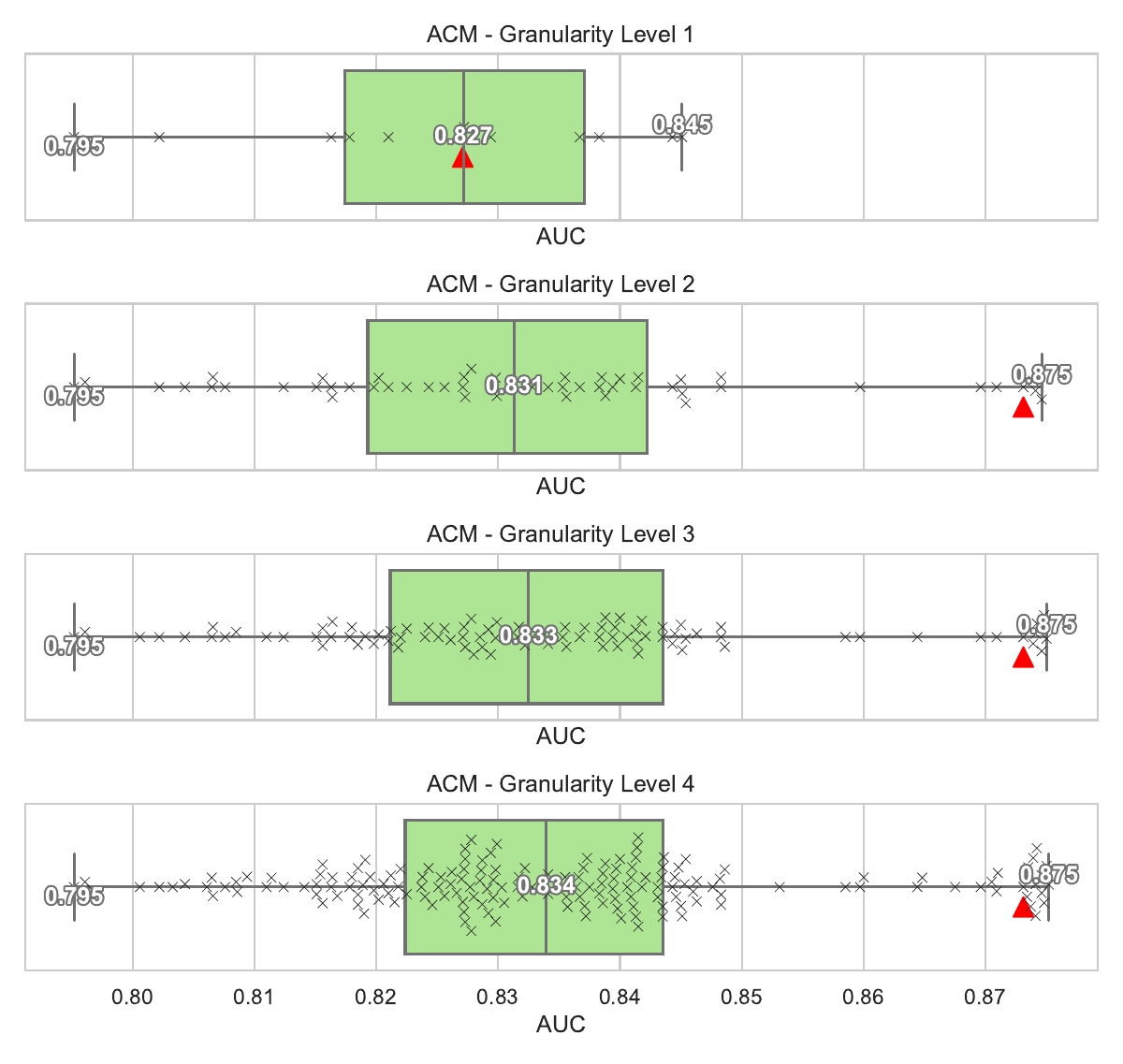}
        \label{fig:acm-gra-level}
    \end{minipage}
    \hfill
    \begin{minipage}[t]{0.32\linewidth}
        \centering
        \includegraphics[width=\linewidth]{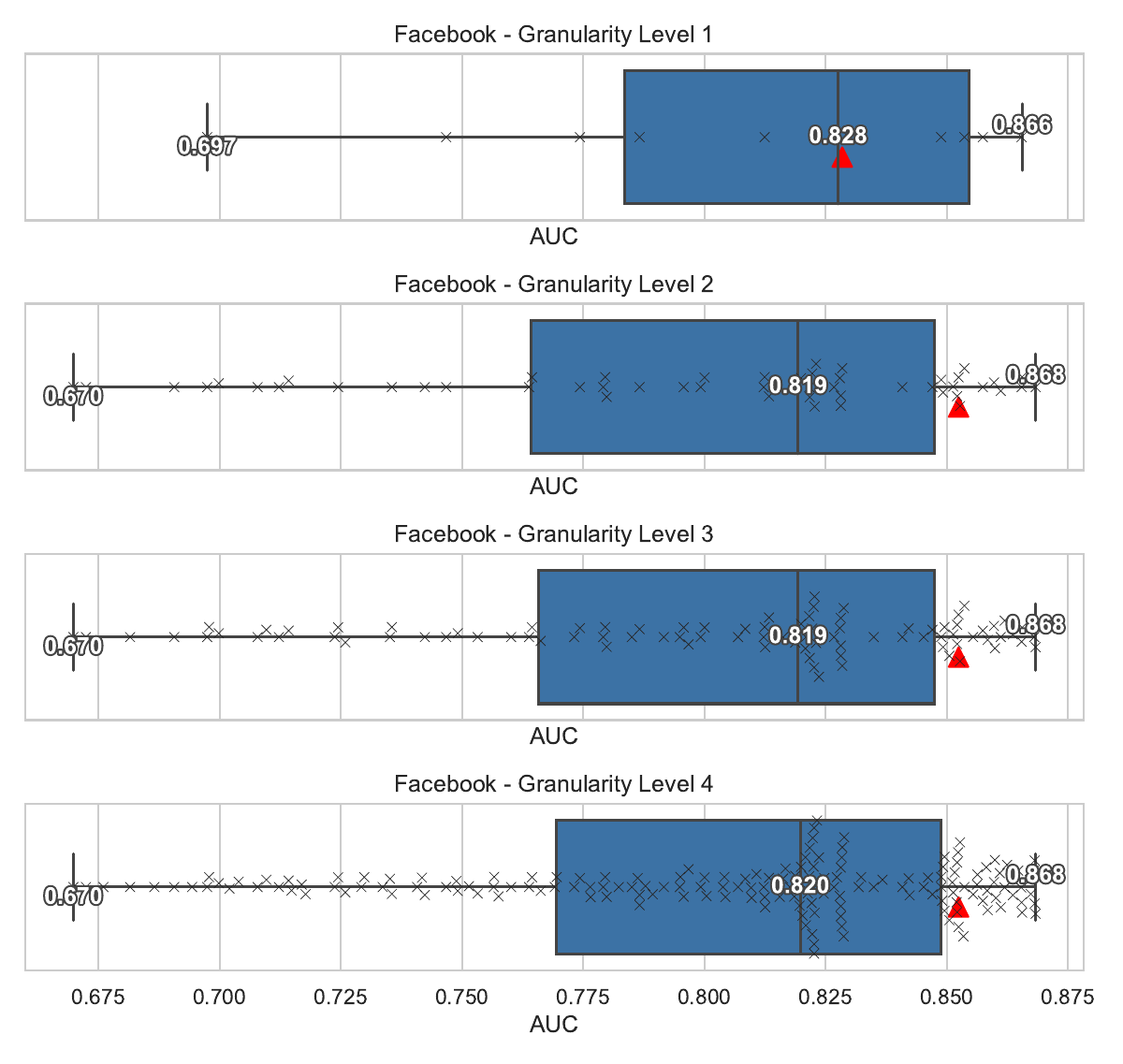}
        \label{fig:facebook-gra-level}
    \end{minipage}
    \caption{{\color{black}Performance of AutoGAD across different granularity levels of search grids using ANEMONE on the Cora, ACM, and Facebook datasets. Similar trends were observed for other anomaly detectors and datasets, which are omitted for brevity.}}
    \label{fig:comparison-gra-level}
\end{figure}

{\color{black}\textbf{Sensitivity to the Granularity of the Search Grid.} Acknowledging the significance of search space granularity in the performance of AutoGAD, we conduct a sensitivity analysis by varying the granularity levels of the search grids in grid search. Figure~\ref{fig:comparison-gra-level} presents representative results using ANEMONE \citep{jin2021anemone} on the Cora, ACM, and Facebook datasets with four levels of search granularity, as follows:
\begin{itemize}
    \item Granularity Level 1: $\alpha \in \{$0, 0.2, 0.4, 0.6, 0.8, 1$\}$, $K \in \{$2, 4$\}$;
    \item Granularity Level 2: $\alpha \in \{$0, 0.01, 0.1, 0.2, 0.3, 0.4, 0.5, 0.6, 0.7, 0.8, 0.9, 0.99, 1$\}$, $K =\{$2, 3, 4, 5$\}$;
    \item Granularity Level 3: $\alpha \in \{$0, 0.01, 0.05, 0.1, 0.15, 0.2, 0.25, 0.3, 0.35, 0.4, 0.45, 0.5, 0.55, 0.6, 0.65, 0.7, 0.75, 0.8, 0.85, 0.9, 0.95, 0.99, 1 $\}, K \in \{$2, 3, 4, 5$\}$;
    \item Granularity Level 4: $\alpha \in \{$0, 0.01, 0.025, 0.05, 0.075, 0.1, 0.125, 0.15, 0.175, 0.2, 0.225, 0.25, 0.275, 0.3, 0.325, 0.35, 0.375, 0.4, 0.425, 0.45, 0.475, 0.5, 0.525, 0.55, 0.575, 0.6, 0.625, 0.65, 0.675, 0.7, 0.725, 0.75, 0.775, 0.8, 0.825, 0.85, 0.875, 0.9, 0.925, 0.95, 0.975, 0.99, 1 $\}, K \in \{$2, 3, 4, 5, 6, 7$\}$.
\end{itemize}
The results indicate that finer search grids tend to improve the performance of AutoGAD. This is expected, as the optimal value achievable in a finer search grid cannot be worse than that in a coarser grid. Similar observations were made for other anomaly detection methods and datasets, which are omitted here for brevity.
}

{\color{black}
\section{Alternative Strategies and Discussion}
Internal evaluation strategies aim to assess the quality of a model based solely on internal information, without relying on external information such as ground-truth labels. Internal information can typically be derived from two sources: 1) the input samples, such as feature values of instances in tabular data or node attributes in graph data; or 2) the anomaly scores generated by an anomaly detection model. Beyond the Contrast Score Margin \citep{xu2019automatic} discussed in this paper, additional internal evaluation strategies exist for unsupervised model selection in anomaly detection. According to \citet{ma2023need}, these strategies can be categorized as  \emph{stand-alone} or \emph{consensus-based} internal evaluation strategies; we will next discuss each category.

\subsection{Stand-alone Internal Evaluation Strategies}
Stand-alone strategies rely solely on input samples or individual anomaly detection methods (or models with specific HP configurations in our setting) and their output anomaly scores. Key methods include:

\begin{itemize}
    \item \textbf{IROES} \citep{marques2015internal,marques2020internal} quantifies the separability of each input sample, assuming that a good anomaly detection model assigns high anomaly scores to highly separable samples. However, separability scores are defined only for tabular data, making extension to graph data non-trivial. Additionally, computing separability scores is computationally expensive, posing challenges for large datasets.
    \item \textbf{Mass-Volume and Excess-Mass} \citep{goix2016evaluate} use statistical tools to measure the quality of an anomaly scoring function. These methods operate on the raw input samples rather than anomaly scores and assume that anomalies occur in the distribution's tail. However, they are restricted to tabular data and are not applicable to graph data.
    \item \textbf{Clustering Validation Metrics} \citep{nguyen2016evaluation} assume that an anomaly detector divides input samples into two clusters: abnormal and normal. Clustering validation metrics, such as the Xie-Beni index \citep{xie1991validity}, are then used to evaluate performance. While clustering coefficients on graphs could be analogous \citep{li2017clustering}, these metrics are computationally expensive, particularly for large datasets.
\end{itemize}

\subsection{Consensus-based Internal Evaluation Strategies}
Consensus-based strategies assess the agreement among multiple anomaly detection models (or the same model with varying HP configurations in our setting). Key methods include:

\begin{itemize}
    \item \textbf{UDR} \citep{duan2019unsupervised} assumes that good HP configurations yield consistent results under different random initializations, while poor configurations do not. \citet{ma2023need} repurposed UDR to select among heterogeneous anomaly detectors, assuming that good detectors produce consistent results across HP configurations.
    \item \textbf{Model Centrality} \citep{lin2020infogan} hypothesizes that good models are close to the optimal model and thus to each other. 
    \item \textbf{Model Centrality by HITS} \citep{kleinberg1999authoritative} follows a similar hypothesis but employs a different computation approach.
    \item \textbf{Unsupervised Anomaly Detection Ensembling} \citep{ma2023need} infers pseudo anomaly labels by aggregating outputs from a predefined subset of good models. However, this method is less feasible in our setting as there is no such pre-defined good models.
\end{itemize}

Two challenges remain when utilizing these strategies in our setting: 1) validating the underlying assumptions, which often lack theoretical justification, and 2) addressing their computational expenses, as consensus-based methods require pairwise comparisons. In contrast, Contrast Score Margin is computationally efficient, as it operates on anomaly scores rather than on raw data points and it avoids pairwise comparisons.

\subsection{Discussion and Future Work}
Although \citet{ma2023need} demonstrated that many internal evaluation strategies perform suboptimally for selecting heterogeneous anomaly detectors, we hypothesize that some can be valuable for hyperparameter tuning within a single anomaly detection model. However, this is beyond the scope of this paper and is left for future work. The primary objectives of this paper are twofold:
\begin{itemize}
    \item We highlight flaws in existing studies on using SSL for unsupervised graph anomaly detection. Specifically, we:
    \begin{enumerate}
        \item Review these studies, showing that most tune HPs arbitrarily or selectively.
        \item Demonstrate empirically, through extensive experiments, that these methods are highly sensitive to HP settings. Consequently, we argue that these methods may suffer from label information leakage under unsupervised learning settings, leading to overstated performance in practical scenarios where label-based tuning is inaccessible.
    \end{enumerate}
    \item We propose an initial solution to these issues by utilizing and improving the Contrast Score Margin. This internal evaluation metric was selected for two reasons:
    \begin{enumerate}
        \item It operates on anomaly scores rather than on raw data points and avoids pairwise computations, making it computationally efficient and suitable for large datasets.
        \item Theoretical guarantees for its properties are provided by Theorem 1, which may not hold for other internal evaluation strategies.
    \end{enumerate}
\end{itemize}

This paper does not aim to provide a perfect solution to the issues mentioned above. Instead, our goal is to spark interest in the research community to address these challenges. Unlike \citet{ma2023need}, we do not aim to conduct a comprehensive review and evaluation of internal evaluation strategies for SSL-based graph anomaly detection, as this requires significant computational resources and in-depth analysis. Nevertheless, we aim to explore this direction in future work by considering and potentially repurposing the internal evaluation strategies reviewed in \citet{ma2023need}. We have described a more advanced search strategy than grid search, namely SMBO-based optimization \citep{jones1998efficient}, in Appendix \ref{appendix:SMBO}, without experimental evaluation. This is because this method introduces additional hyperparameters and their tuning is non-trivial, contradicting our goal of automated anomaly detection. Other advanced hyper-parameter tuning methods \citep{yang2020hyperparameter,bischl2023hyperparameter,zhao2024hpod} to speed up the search are possible, and we leave their explorations for future work.
}

\section{Conclusions}
SSL has received much attention in recent years, and many recent studies have explored SSL to perform unsupervised GAD. However, we found that most existing studies tune hyperparameters arbitrarily or selectively (i.e., guided by labels), and our empirical findings reveal that most methods are highly sensitive to hyperparameter settings. Using label information to tune hyperparameters in an unsupervised setting, however, is label information leakage and leads to severe overestimation of model performance. To mitigate this issue, we introduce AutoGAD, the first automated hyperparameter selection method for SSL-based unsupervised GAD. Extensive experiments demonstrate the effectiveness of our proposed strategy. Overall, we aim to raise awareness to the label information leakage issue in the unsupervised GAD field, and AutoGAD provides a first step towards achieving truly unsupervised SSL-based GAD.

\section*{Statement and Declaration}
\backmatter


\bmhead{Ethical approval} This study does not involve human and animal data, and thus the need for approval was waived.

\bmhead{Funding}
This work is supported by Project 4 of the Digital Twin research programme, a TTW Perspectief programme with project number P18-03 that is primarily financed by the Dutch Research Council (NWO). All opinions, findings, conclusions and recommendations in this paper are those of the authors and do not necessarily reflect the views of the funding agencies.

\bmhead{Conflict of interest} The author(s) declared no potential conflicts of interest with
respect to the research, authorship and/or publication of this
article.

\bmhead{Availability of data and materials}
For reproducibility, all code and datasets are provided online via the following link: \href{https://github.com/ZhongLIFR/AutoGAD2024}{https://github.com/ZhongLIFR/AutoGAD2024}.

\bmhead{Authorship Contribution}

\textit{Zhong Li}: Conceptualization, Methodology, Validation, Investigation, Software, Writing, Visualisation, Project Administration. \textit{Yuhang Wang}: Methodology, Investigation, Software, Writing. \textit{Matthijs van Leeuwen}: Methodology, Validation, Writing, Funding acquisition, Supervision.

\bibliography{References}

\appendix

\section{Pitfalls in Existing Methods (Full Analysis)}
\label{appendix:pitfalls}
\subsection{CoLA} Particularly, CoLA \citep{liu2021anomaly} is the first \textit{contrastive-based} framework for unsupervised GAD. The design of its \textit{data augmentation} module and \textit{contrast learning} module is as follows.

\textbf{Data Augmentation Module} They consider one type of data augmentation, subgraph sampling, to obtain local augmented view for each node. Particularly, they employ RWR \citep{tong2006fast} to generate a sub-graph with a fixed size $K$ in subgraph sampling, resulting in one critical HP in graph augmentation, namely $K$.

\textbf{Contrast Learning Module} They consider a single contrast aspect, namely node-subgraph contrast between the embedding of the target node and the aggregated embedding of its local sug-graph, without resulting in any HPs.

\textbf{HPs Sensitivity \& Tuning}
They conducted sensitive analysis and found that the selection of subgraph size $K$ is dependent on the specific dataset. The AUC performance usually increases first and then decreases with the increasing of $K$. However, for efficiency and robustness consideration, they heuristically set the sampled subgraph size $K=4$ for all datasets. 

\subsection{ANEMONE} ANEMONE \citep{jin2021anemone} is a \textit{contrastive-based} framework for unsupervised GAD. They argue that modeling the relationships in a single contrastive perspective leads to limited capability of capturing complex anomalous patterns, and thus propose additional contrast perspectives as follows.

\textbf{Graph Augmentation Module} They consider a single graph augmentation operation, namely Random Ego-Nets generation with a fixed size $K$. Specifically, taking the target node as the center, they employ RWR \citep{tong2006fast} to generate two different subgraphs as ego-nets with a fixed size $K$. Overall, they result in one critical HP in graph augmentation, namely $K$.

\textbf{Contrast Learning Module} They consider two contrast perspectives: 1) node-node contrast between the embedding of masked target node within ego-net and the embedding of the original node, leading to loss term $\mathcal{L}_{NN}$, and 2) node-subgraph contrast within each view, leading to loss term $\mathcal{L}_{NS}$. On this basis, they combine these loss terms as $$\mathcal{L} = (1-\alpha)\mathcal{L}_{NN}+\alpha\mathcal{L}_{NS}$$ where $\alpha \in [0,1]$ is the trade-off HP. Hence, they result in one critical HP in graph contrast, namely $\alpha$.

\textbf{HPs Sensitivity \& Tuning} In their ablation studies: 1) by using ground-truth label information, they heuristically set $\alpha$ as $0.8, 0.6, 0.8$ on Cora, CiterSeer and PubMed respectively,  and report the corresponding results; and 2) the setting of $K$ was not studied, and it is set to 4 for all datasets.

\subsection{GRADATE} GRADATE \citep{duan2023graph} is also a \textit{contrastive-based} framework. They argue that subgraph-subgraph contrast is also critical in detecting graph anomalies, and design it as follows.

\textbf{Data Augmentation Module} They consider a single graph augmentation operation, namely Edge Modification that removes edges in the adjacency matrix as well as add the same number of edges. Concretely, they fix a proportion $P$, and then uniformly and randomly sample $\frac{P\cdot M}{2}$ edges from a total of $M$ edges to remove. Meanwhile, $\frac{P\cdot M}{2}$ edges are added into the adjacency matrix. Overall, they result in one critical HP in graph augmentation, namely $P$.

\textbf{Contrast Learning Module} They consider three contrast aspects: 1) node-node contrast within each view, leading to loss term $\mathcal{L}_{NN}$), 2) node-subgraph contrast within each view, leading to loss term $\mathcal{L}_{NS}$, and 3) subgraph-subgraph contrast between original view and augmented view, leading to loss term $\mathcal{L}_{SS}$. On this basis, they combine these loss terms as $$\mathcal{L} = (1-\beta)\mathcal{L}_{NN}+\beta\mathcal{L}_{NS}+\gamma\mathcal{L}_{SS},$$ where $\beta,\gamma \in (0,1)$ are trade-off HPs. More, $\mathcal{L}_{NN} = \alpha\mathcal{L}_{NN,1} + (1-\alpha)\mathcal{L}_{NN,2}$, and $\mathcal{L}_{NS} = \alpha\mathcal{L}_{NS,1} + (1-\alpha)\mathcal{L}_{NS,2}$, with $\mathcal{L}_{NN,1}$ and $\mathcal{L}_{NN,2}$ being the loss term in the first and second views respectively. Overall, they result in three critical HPs in graph contrast, namely the combination weights $\alpha, \beta, \gamma$.

\textbf{HPs Sensitivity \& Tuning} In their ablation studies, 1) they compared four different graph augmentation strategies, including Gaussian Noise Feature, Feature Masking, Graph Diffusion, and Edge Modification, and they found that Edge Modification performs the best across different datasets (with ground-truth labels on test data to measure the performance); 2) with the help of ground-truth label information on test data, they heuristically set $(\alpha,\beta)$ as $(0.9,0.3), (0.1,0.7), (0.7,0.1), $ $(0.9,0.3), (0.7,0.5), (0.5,0.5)$ on EAT, WebKB, UAT, Cora, UAI2010, and Citation respectively; 3) similarly, they set $\gamma=1$ for all datasets; and 4) they also heuristically set $P=0.2$ for all datasets.

\subsection{SL-GAD} Different from CoLA, ANEMONE and GRADATE, SL-GAD \citep{zheng2021generative} combines the \textit{contrastive-based} framework and the \textit{generative-based} framework for unsupervised GAD.

First, the design of the \textit{contrastive-based} framework  is as follows.

\textbf{Contrastive Framework---Data Augmentation Module}
They consider a single graph augmentation operation, namely Random Ego-Nets generation with a fixed size $K$. Specifically, taking the target node as the center, they employ RWR \citep{tong2006fast} to generate two different subgraphs as ego-nets with a fixed size $K$, where $K$ controls the radius of the surrounding contexts. Overall, they result in one critical HP in graph augmentation, namely $K$. Particularly, they indicate that other augmentation strategies such as attribute masking and edge modification may introduce extra anomalies, while random ego-nets and graph diffusion can augment data without changing the underlying graph semantic information.

\textbf{Contrastive Framework---Contrast Learning Module} They introduce a Multi-View Contrastive Learning module that compare the similarity between node embedding and embedding of sampled sub-graphs in augmented views (namely node-subgraph contrast), leading to two loss terms $\mathcal{L}_{con,1}$ and $\mathcal{L}_{con,2}$ corresponding to two augmented views, respectively. On this basis, they obtain the contrastive objective $\mathcal{L}_{con} = \frac{1}{2}(\mathcal{L}_{con,1}+\mathcal{L}_{con,2})$, which combines the two loss terms with equal weights.

Second,  the \textit{generative-based} framework  is designed as follows.

\textbf{Generative Framework} They introduce a Generative Attribute Regression module that reconstructs node attributes, with the aim to achieve node-level discrimination, where the encoder is a GCN and the decoder is another GCN. Specifically, they minimize the Mean Square Error between the target node's original and reconstructed attributes in augmented views, leading to two loss terms $\mathcal{L}_{gen,1}$ and $\mathcal{L}_{gen,2}$ corresponding to two augmented views, respectively. Then they combine them with equal weights, leading to the generative objective $\mathcal{L}_{gen} = \frac{1}{2}(\mathcal{L}_{gen,1}+\mathcal{L}_{gen,2})$.

At last, their final optimization objective is defined as  follows:
$$\mathcal{L} = \alpha\mathcal{L}_{con}+\beta\mathcal{L}_{gen},$$
where $\alpha,\beta \in (0,1]$ are trade-off HPs to balance the importance of two SSL objectives.

\textbf{HPs Sensitivity \& Tuning } They conducted sensitive analysis and found that: 1) the performance first increases and then decreases with the increasing of $K$. For efficiency consideration, they heuristically set the sampled subgraph size $K=4$ for all datasets; 2) they heuristically fix $\alpha = 1$ for all datasets as they found that this achieves good performance on most datasets (with the help of label information); and 3) the selection of $\beta$ is high dependent on the specific dataset. Hence, they ``fine-tune" the value of $\beta$ for each dataset via selecting $\beta$ from $\{0.2, 0.4, 0.6, 0.8, 1.0\}$ with labels.

\subsection{Sub-CR} Similar to SL-GAD, Sub-CR \citep{zhang2022reconstruction} also combines the \textit{contrastive-based} framework and the \textit{generative-based} framework for unsupervised GAD.

First, the design of the \textit{contrastive-based} framework  is as follows.

\textbf{Contrastive Framework---Contrast Learning Module} They consider two types of data augmentation: 1) subgraph sampling to obtain local augmented views for each node (so-called local view subgraph), 2) graph diffusion plus subgraph sampling (in a sequential order) to obtain global augmented views for each node (so-called global view subgraph). Particularly, they employ RWR \citep{tong2006fast} to generate a sub-graph with a fixed size $K$ in subgraph sampling. Besides, they apply Persnonalized PageRank to power the graph diffusion \citep{zhang2023survey}, wherein the teleport probability $\alpha$ needs to be determined. Overall, they result in two critical HPs in graph augmentation, namely $K$ and $\alpha$.

\textbf{Contrastive Framework---Contrast Learning Module} This module consists of: 1) intra-view contrastive learning that maximizes the agreement between the node and its sub-graph level representations in the local view (with loss term $\mathbf{L}_{intra,1}$), and the agreement between the node and its sub-graph level representations in the global view (with loss term $\mathbf{L}_{intra,2}$), where they combine the local view and global view loss terms with equal weights to obtain the intra-view loss term $\mathbf{L}_{intra} = \mathbf{L}_{intra,1} + \mathbf{L}_{intra,2}$; and 2) inter-view contrastive learning that makes closer the discriminative scores of node-subgraph pairs in local view and global view, leading to the loss term $\mathbf{L}_{inter}$. On this basis, they combine the intra-view loss term and  inter-view loss term with equal weights to obtain the multi-view contrastive learning loss term $\mathbf{L}_{con} = \mathbf{L}_{intra} + \mathbf{L}_{inter}$. 

Second,  the \textit{generative-based} framework  is designed as follows.

\textbf{Generative Framework} They introduce a masked Autoencoder-based Reconstruction module, where the encoder is a GCN and the decoder is a multilayer perceptron with \textit{PReLU} activation function, aiming to reconstruct the attributes of the target node based on the attributes of neighboring nodes in the local view (with loss term $\mathbf{L}_{res,1}$), and in the global view (with loss term $\mathbf{L}_{res,2}$). Next, they combine the local view and global view loss terms with equal weights to obtain the overall reconstruction loss term $\mathbf{L}_{res} = \mathbf{L}_{res,1} + \mathbf{L}_{res,2}$ for each node.

At last, their final optimisation objective is defined as  follows:
$$\mathcal{L} = \mathcal{L}_{con}+\gamma\mathcal{L}_{res},$$
where $\gamma \in (0,1]$ is the trade-off HP to balance the importance of two different SSL objectives.

\textbf{HPs Sensitivity \& Tuning} They conducted sensitive analysis and found that: 1) the selection of $K$ is dependent on the specific dataset. However, for efficiency and performance consideration, they heuristically set the sampled subgraph size $K=4$ for all datasets; 2) they did not discuss the setting of teleport probability $\alpha$; and 3) they claim that most datasets are not sensitive to the value of $\gamma$ when $\gamma>0.4$. Hence, they heuristically set $\gamma = 0.6$ for Cora, Citeseer, Flickr, and BlogCatalog while $\gamma = 0.4$ for PubMed with the help of label information.

\subsection{CONAD} Similar to SL-GAD and Sub-CR, CONAD \citep{xu2022contrastive} also combines the \textit{contrastive-based} framework and the \textit{generative-based} framework for unsupervised GAD.

First, the design of the \textit{contrastive-based} framework  is as follows.

\textbf{Contrastive Framework---Data Augmentation Module} They consider four different types of data augmentations, with each type of data augmentation operation corresponding to a specific type of node anomaly. They include 1) edge adding augmentation that connects a node with many other non-connected nodes (structure - high degree), 2) edge removing augmentation that removes most edges of a node (structure - outlying); 3) attribute replacement augmentation that replaces the target node's attributes with another dissimilar node's attributes (attribute - deviated), and 4) attribute scaling augmentation that scales the target node's attributes to much larger or smaller values (attribute - disproportionate); This leads to four HPs $p_{1},p_{2},p_{3},p_{4}$, which represent the sampling probability of each augmentation strategy. Moreover, the rate $r$ of augmented anomalies (namely modified nodes) is also a HP.

\textbf{Contrastive Framework---Contrast Learning Module} They consider two different contrast strategies: 1) Siamese contrast $\mathcal{L}_{SC} = \sum_{i \in \mathbf{NM}} d(\mathbf{z}_{i},\hat{\mathbf{z}}_{i}) +\sum_{j \in \mathbf{MM}} max\{0, m-d(\mathbf{z}_{j},\hat{\mathbf{z}}_{j})\}$ where $d(\mathbf{z}_{i},\hat{\mathbf{z}}_{i})$ is the distance between embeddings of node $i$ in the original view and in the augmented view. $\mathbf{MM}$ and $\mathbf{NM}$ mean the node is modified or non-modified, respectively; 2) Triplet contrast $\mathcal{L}_{TC} = \sum max\{0, m-[d(\mathbf{z}_{i},\hat{\mathbf{z}}_{j})-d(\mathbf{z}_{i},\mathbf{z}_{j})]\}$ where $d(\mathbf{z}_{i},\mathbf{z}_{j})$ is the distance between embeddings of node $i$ and its neighbor $j$ in the original view, and $d(\mathbf{z}_{i},\hat{\mathbf{z}}_{j})$ is the distance between embeddings of node $i$ in the original view and its neighbor $j$ in the augmented view. Particularly, the contrastive loss term $\mathcal{L}_{Contr} = \mathcal{L}_{SC}$ or $\mathcal{L}_{Contr} = \mathcal{L}_{TC}$. This module contains a HP, namely the margin $m$.

Second,  the \textit{generative-based} framework  is designed as follows.

\textbf{Generative Framework} This framework consists of two components: 1) an attribute autoencoder to reconstruct the node attributes, where the encoder is a GAT \citep{velivckovic2017graph} and the decoder is another GAT. This leads to the loss term $L_{A}$; and 2) a structure autoencoder to reconstruct the structure, where the encoder is a GAT and the decoder is a dot product operation followed by a \textit{sigmoid} function (namely $sigmoid(\mathbf{z}^{t}\mathbf{z})$). This leads to the loss term $L_{S}$. Combining these two loss terms leads to a loss term $L_{Recon} = \lambda L_{A}+ (1-\lambda) L_{S}$, where $\lambda \in (0,1)$  is a trade-off HP to balance the two reconstruction errors. Unlike SL-GAD and Sub-CR, CONAD requires the whole adjacency matrix and node attribute matrix as input, and thus it can reconstruct the graph structure, making it unsuitable to large graphs. In contrast, SL-GAD and Sub-CR only require subgraphs as inputs, and thus are unable to perform structure reconstruction while being scalable.

At last, the final optimization objective is defined as  follows:
$$\mathcal{L} = \eta\mathcal{L}_{Contr}+(1-\eta)\mathcal{L}_{Recon},$$
where $\eta \in (0,1)$ is the trade-off HP to balance the importance of two SSL objectives.

\textbf{HPs Sensitivity \& Tuning} They did not perform sensitivity analysis over the HPs. Instead,
1) They heuristically set the ration of augmented anomalies  $r=0.1$ and $r=0.2$ for small and large datasets, respectively; 2) The sampling probability of each augmentation strategy is set to $p_{i}=0.25$ for $i \in \{1,2,3,4\}$; 3) They heuristically set the margin $m=0.5$ for all datasets; and 4) They heuristically set the trade-off hyper-parameters $\lambda=0.9$ and  $\eta=0.7$ for all datasets

\subsection{DOMINANT} DOMINANT \citep{ding2019deep} is arguably the first work that utilizes \textit{generative-based} framework and GNNs to perform unsupervised anomaly detection on attribute graphs.

\textbf{Generative Framework} They first employ GCN \citep{kipf2016semi} to obtain node embeddings. Next, they construct two decoders: 1) an attribute decoder, which consists of another GCN, to reconstruct the node attributes, leading to the loss term $L_{A}$, and 2) a structure decoder, which is a dot product operation followed by a \textit{sigmoid} function (namely $sigmoid(\mathbf{z}^{t}\mathbf{z})$), to reconstruct topological structures, leading to the loss term $L_{S}$.

At last, their final optimization objective is defined as  follows:
$$\mathcal{L} = \alpha\mathcal{L}_{A}+(1-\alpha)\mathcal{L}_{S},$$
where $\alpha \in (0,1)$ is the trade-off HP to balance the importance of two objectives.

\textbf{HPs Sensitivity \& Tuning} Specifically, they found that the AUC performance usually increases first and then decreases with the increasing of $\alpha$. However, the specific value of $\alpha$ on each dataset is heuristically selected with the help of labels. The HP $\alpha$ is selected from $[0.4,0.7], [0.4,0.7], [0.5,0.8]$ on BlogCatalog, Flickr, and ACM respectively.

\subsection{AnomalyDAE} Similar to DOMINANT, AnomalyDAE \citep{fan2020anomalydae} leverages \textit{generative-based} framework and autoencoders (based on GNNs) to perform unsupervised GAD.

\textbf{Generative Framework} AnomalyDAE consists of two components: 1) an attribute autoencoder to reconstruct the node attributes, where the encoder consists of two non-linear feature transform
layers and the decoder is simply a dot product operation. This leads to the loss term $L_{A}$, and $L_{A}$ is associated with a penalty HP $\eta >1$); and 2) a structure autoencoder to reconstruct the structures, where the encoder is based GAT \citep{velivckovic2017graph} and the decoder is a dot product operation followed by a \textit{sigmoid} function (namely $sigmoid(\mathbf{z}^{t}\mathbf{z})$). This leads to the loss term $L_{S}$, and $L_{S}$ is associated with a penalty HP $\theta >1$.

At last, their final optimization objective is defined as  follows:
$$\mathcal{L} = \alpha\mathcal{L}_{S}+(1-\alpha)\mathcal{L}_{A},$$
where $\alpha \in (0,1)$ is the trade-off HP to balance the importance of two objectives.

\textbf{HPs Sensitivity \& Tuning} Specifically, they found that the AUC performance usually increases first and then decreases with the increasing of $\alpha$. However, the specific value of $\alpha$ on each dataset is selected using label information. The HPs $(\alpha, \eta, \theta)$ are heuristically set as $(0.7,5,40)$, $(0.9,8,90)$, $(0.7,8,10)$ on BlogCatalog, Flickr, and ACM respectively.

\subsection{GUIDE} Similar to AnomalyDAE, GUIDE \citep{yuan2021higher} leverages \textit{generative-based} framework and autoencoders (based on GNNs) to perform unsupervised GAD. Particularly, they consider reconstructing the high-order structures.

\textbf{Generative Framework} GUIDE consists of two components: 1) an attribute autoencoder to reconstruct the node attributes, where the encoder is a GCN and the decoder is another GCN. This leads to  the loss term $L_{A}$; and 2) a structure autoencoder to reconstruct the high-order structures, where the encoder is a graph node
attention network based on \citep{ding2021inductive} and the decoder is another graph node attention layer. This leads to the loss term $L_{S}$. Moreover, structure matrix is composed of node motif degrees, which leads to a HP, namely the degree of motifs $D$.

At last, their final optimization objective is defined as  follows:
$$\mathcal{L} = \alpha\mathcal{L}_{A}+(1-\alpha)\mathcal{L}_{S},$$
where $\alpha \in (0,1)$ is the trade-off HP to balance the importance of two SSL objectives.

\textbf{HPs Sensitivity \& Tuning} They mention that the HPs are optimised via a parameter sensitivity analysis experiment for each dataset. Specifically, they found that: 1) the AUC performance usually increases first and then decreases with the increasing of $\alpha$, and most datasets can achieve a good performance when $0.1<\alpha<0.3$. However, the specific value of $\alpha$ on each dataset is selected using labels; and 2) they heuristically set the degree of motifs as $D=4$.

\subsection{GAAN} GAAN \citep{chen2020generative} combines  the \textit{generative-based} framework and GAN \citep{goodfellow2014generative} for unsupervised GAD. Particularly, GAN can be considered as a special case of \textit{contrastive-based} framework.

\textbf{Contrastive Framework---Data Augmentation Module}  GAAN  employs GAN, which consists of a generator and a discriminator, to generate adversarial samples as augmented views, without involving any HPs.

\textbf{Contrastive Framework---Contrastive Learning Module}  For each target node, GAAN computes the sum of cross-entropy losses of its 1-hop neighboring nodes (where the edge is considered as from real distribution by the discriminator) as anomaly score, leading to a loss term $\mathcal{L}_{D}$. In particular, this discriminator loss can be regarded as contrastive loss, and it considers both node attributes and graph structures.

\textbf{Generative Framework} GAAN utilizes the generator to reconstruct the node attribute, and employs the reconstruction error to compute anomaly score, leading to a loss term $\mathcal{L}_{G}$.

At last, their final optimisation objective is defined as
$$\mathcal{L} = \alpha\mathcal{L}_{G}+(1-\alpha)\mathcal{L}_{D},$$
where $\alpha \in [0,1]$ is the trade-off HP to balance the importance of two objectives

\textbf{HPs Sensitivity \& Tuning} Specifically, they found that the AUC performance usually increases first and then decreases with the increasing of $\alpha$. However, the specific value of $\alpha$ on each dataset is selected using label information. The HP $\alpha$ is heuristically set as $0.2, 0.3, 0.1$ on BlogCatalog, Flickr, and ACM respectively.

\section{Performance Variations under Different HP Settings}
\label{appendix:FurtherAnalysis}
In this section, we present a comprehensive analysis of the performance exhibited by various semi-supervised learning (SSL) based graph anomaly detection techniques. This evaluation encompasses an extensive array of hyperparameter (HP) configurations and is conducted across multiple benchmark datasets. 

Specifically, the results for GAAN \citep{chen2020generative} is provided in Figure~\ref{Fig:ROC_Var_GAAN}, from which we can see huge performance variations under different HP settings. For example, the AUC value can vary from $0.474$ to $0.747$ if one utilizes different HP configurations on dataset CiteSeer (namely by changing the HP $\alpha$ from $0.5$ to $0$). Moreover, the results for CoLA \citep{liu2021anomaly} is provided in Figure~\ref{Fig:ROC_Var_CoLA}. Compared to GAAN, CoLA is less sensitive to the setting of HPs, while we can still see moderate performance variations on some datasets (e.g., from $0.693$ to $0.733$ on Flickr, and from $0.767$ to $0.795$ on ACM). Besides, Figure~\ref{Fig:ROC_Var_DOMINANT} shows that DOMINANT is also sensitive to HPs except for the cases where the algorithm is largely underfitted (i.e., on ACM, Flickr and BlogCatalog the loss values change only by $10^{-2}$ after 400 epochs of training).

Particularly, AnomalyDAE and SL-GAD are very sensitive to HPs as shown in Figures~\ref{Fig:ROC_Var_AnomalyDAE} and ~\ref{Fig:ROC_Var_SL-GAD}. For example, the performance of AnomalyDAE ranges from $0.702$ to $0.941$ on CiteSeer, and the performance of SL-GAD vary from  $0.787$ to $0.920$. As shown in Figure~\ref{Fig:ROC_Var_CONAD}, CONAD shows similar behaviors except for the cases where CONAD is largely underfitted (namely on ACM) or suffers from OOM errors (namely on Flickr and BlogCatalog). The analysis for GUIDE in Figure~\ref{Fig:ROC_Var_GUIDE}, GRADATE in Figure~\ref{Fig:ROC_Var_GRADATE}, and Sub-CR in Figure~\ref{Fig:ROC_Var_Sub-CR}
is similar and conveys the same issues.

\begin{figure}
\centering
\includegraphics[width=13cm]{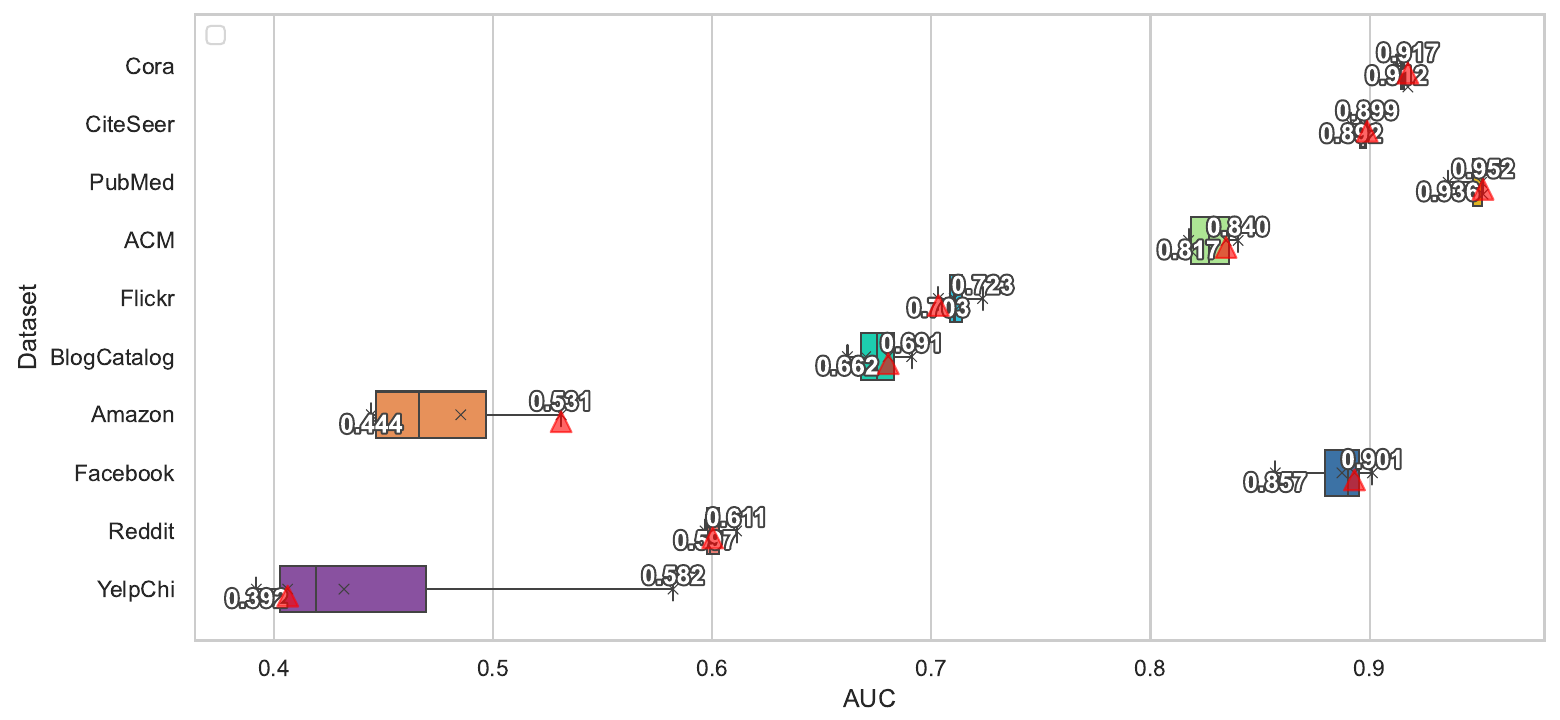}
\caption{Performance variations over different HP configurations for CoLA \citep{liu2021anomaly} on different benchmark datasets. }
\label{Fig:ROC_Var_CoLA}
\end{figure}

\begin{figure}
\centering
\includegraphics[width=13cm]{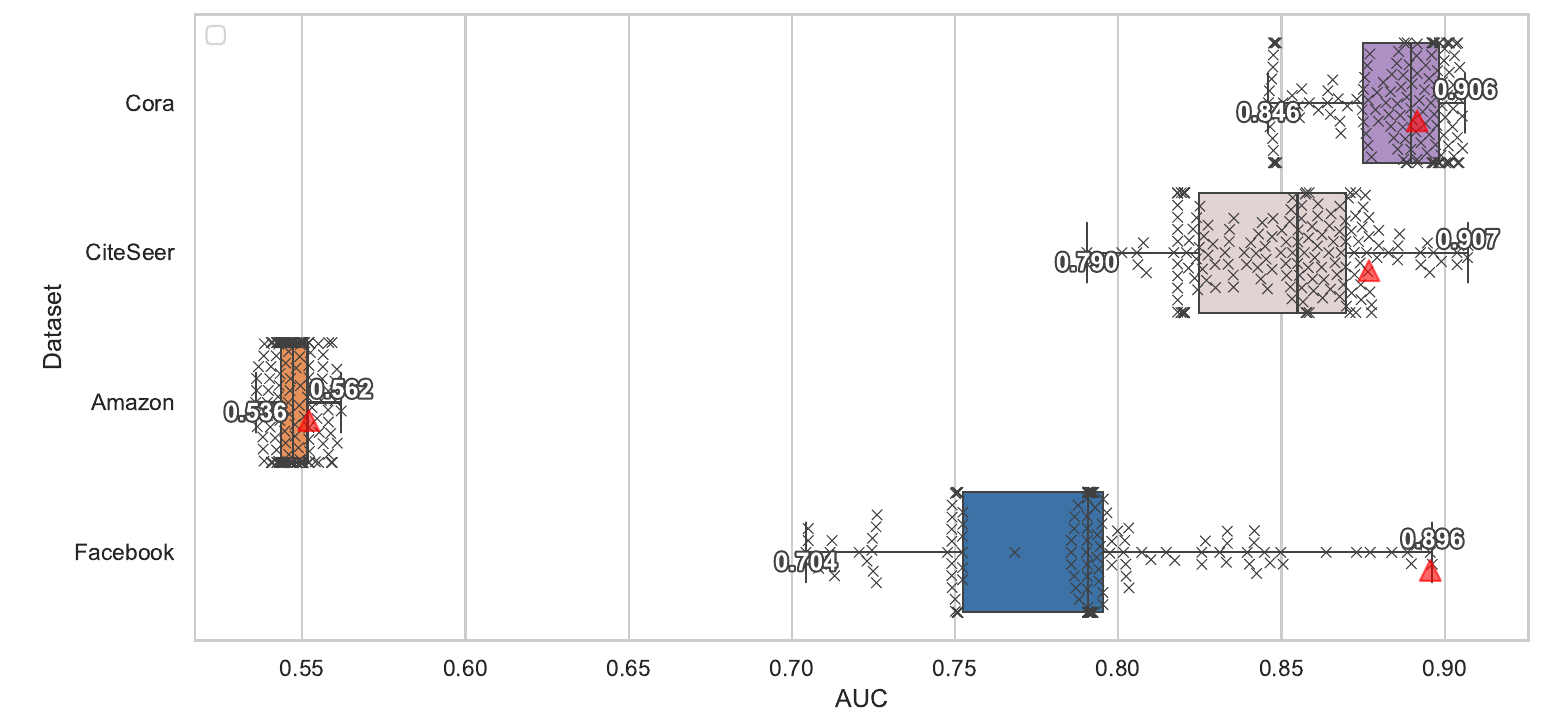}
\caption{Performance variations over different HP configurations for GRADATE \citep{duan2023graph} on different benchmark datasets. }
\label{Fig:ROC_Var_GRADATE}
\end{figure}

\begin{figure}
\centering
\includegraphics[width=13cm]{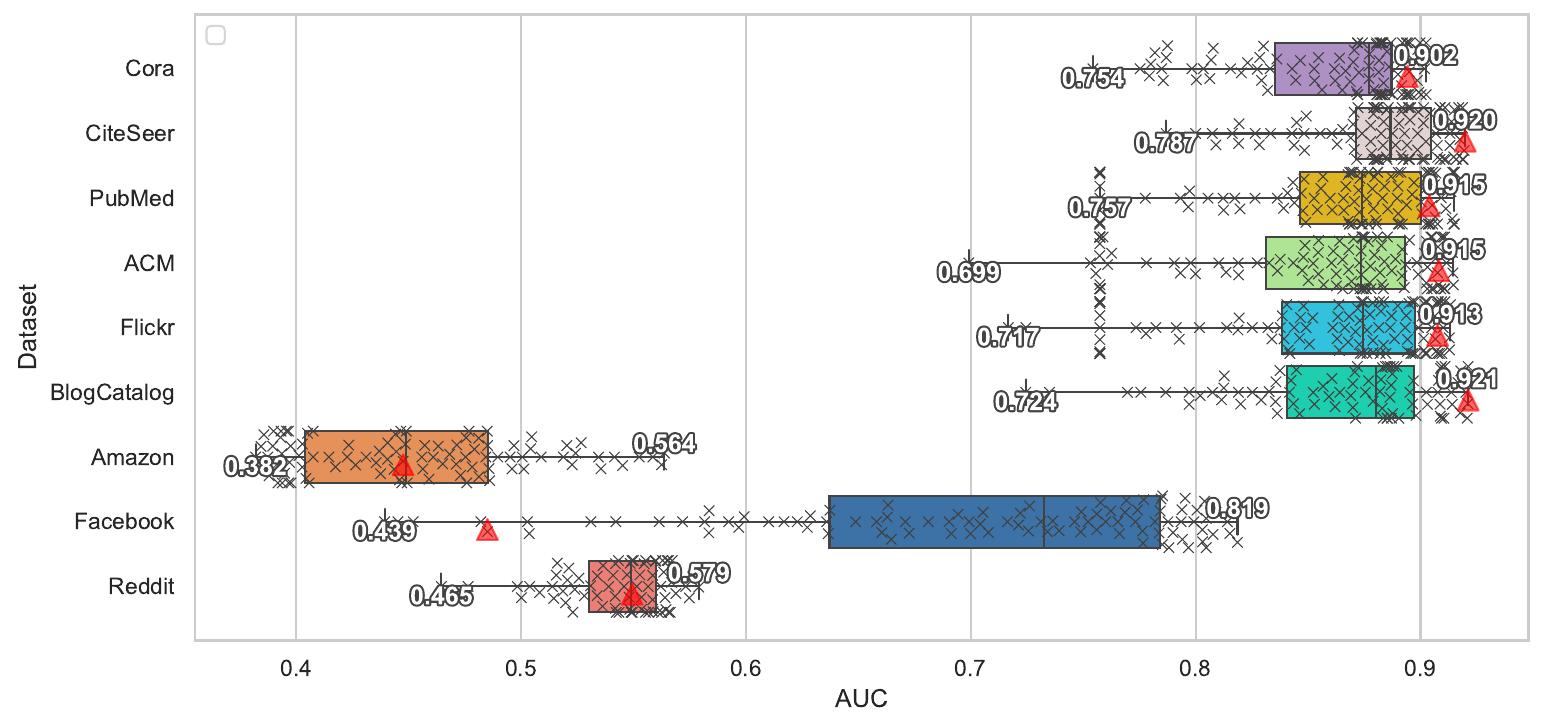}
\caption{Performance variations over different HP configurations for SL-GAD \citep{zheng2021generative} on different benchmark datasets. }
\label{Fig:ROC_Var_SL-GAD}
\end{figure}

\begin{figure}
\centering
\includegraphics[width=13cm]{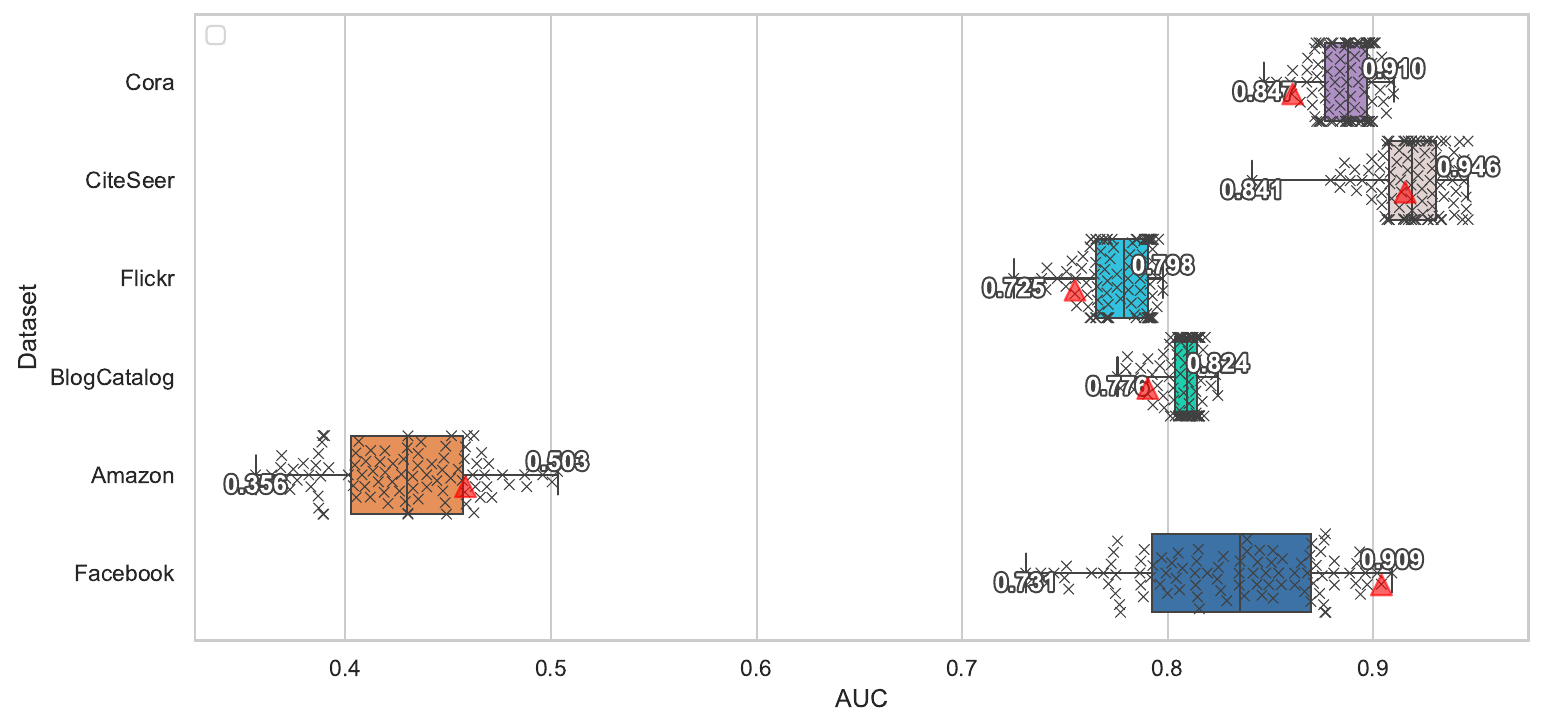}
\caption{Performance variations over different HP configurations for Sub-CR \citep{zhang2022reconstruction} on different benchmark datasets. }
\label{Fig:ROC_Var_Sub-CR}
\end{figure}

\begin{figure}
\centering
\includegraphics[width=13cm]{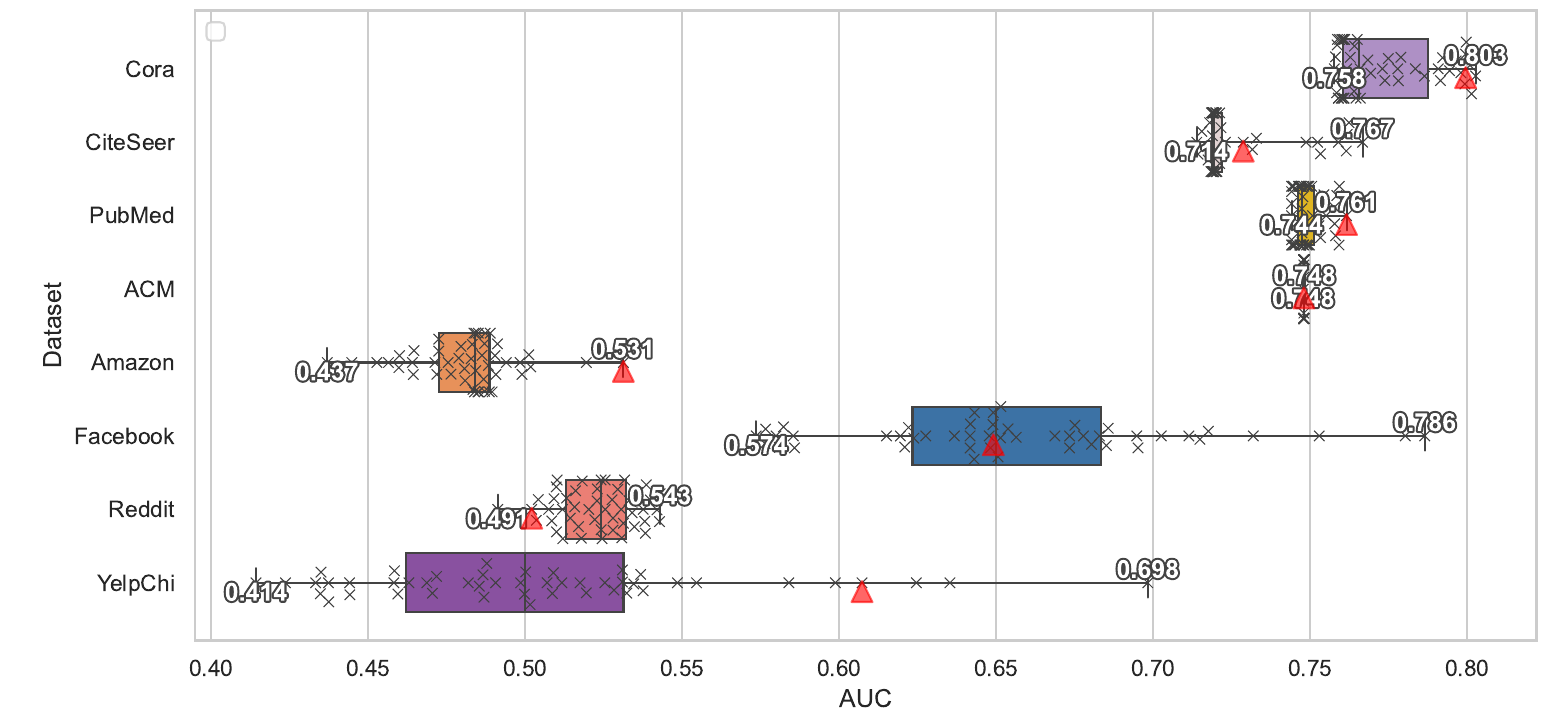}
\caption{Performance variations over different HP configurations for CONAD \citep{xu2022contrastive} on different benchmark datasets. }
\label{Fig:ROC_Var_CONAD}
\end{figure}

\begin{figure}
\centering
\includegraphics[width=13cm]{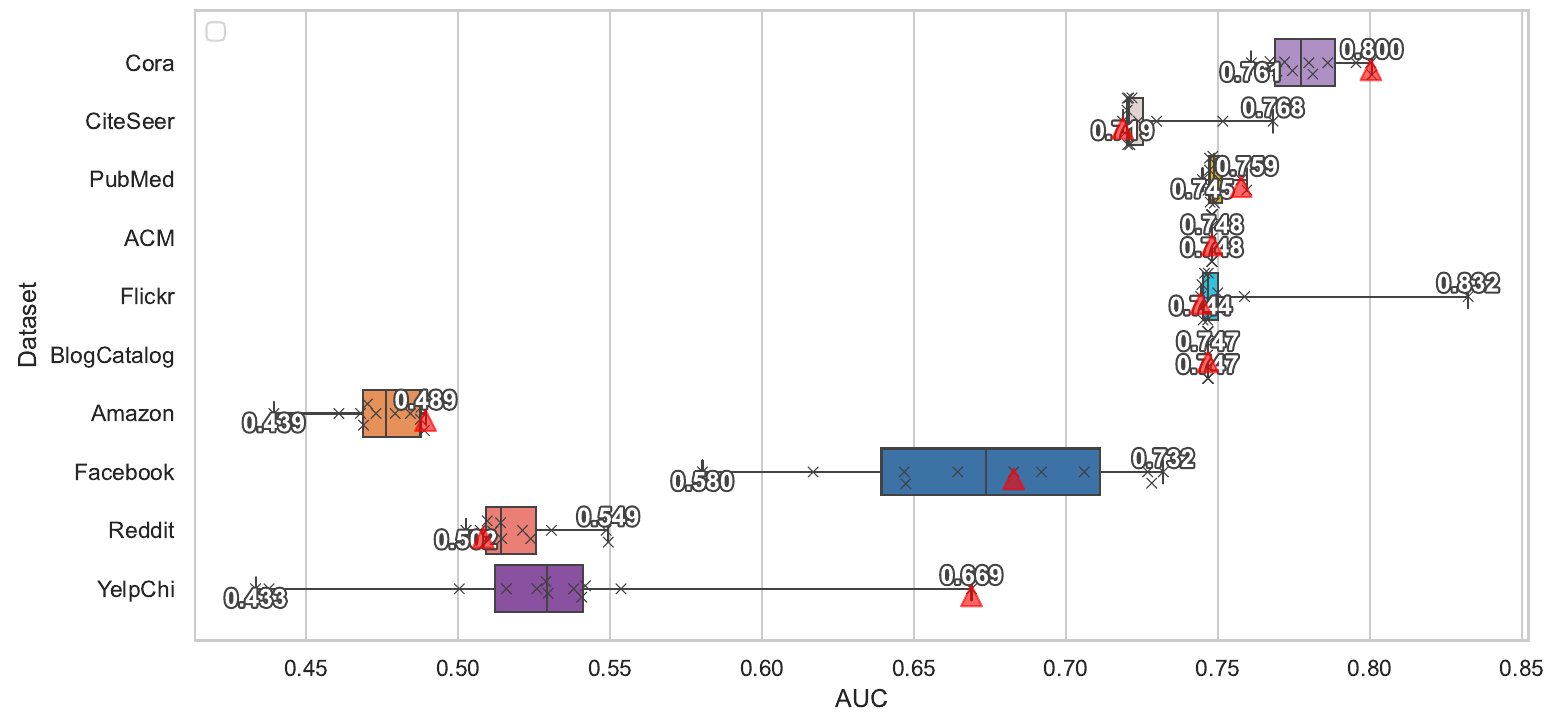}
\caption{Performance variations over different HP configurations for DOMINANT \citep{ding2019deep} on different benchmark datasets. }
\label{Fig:ROC_Var_DOMINANT}
\end{figure}

\begin{figure}[H]
\centering
\includegraphics[width=13cm]{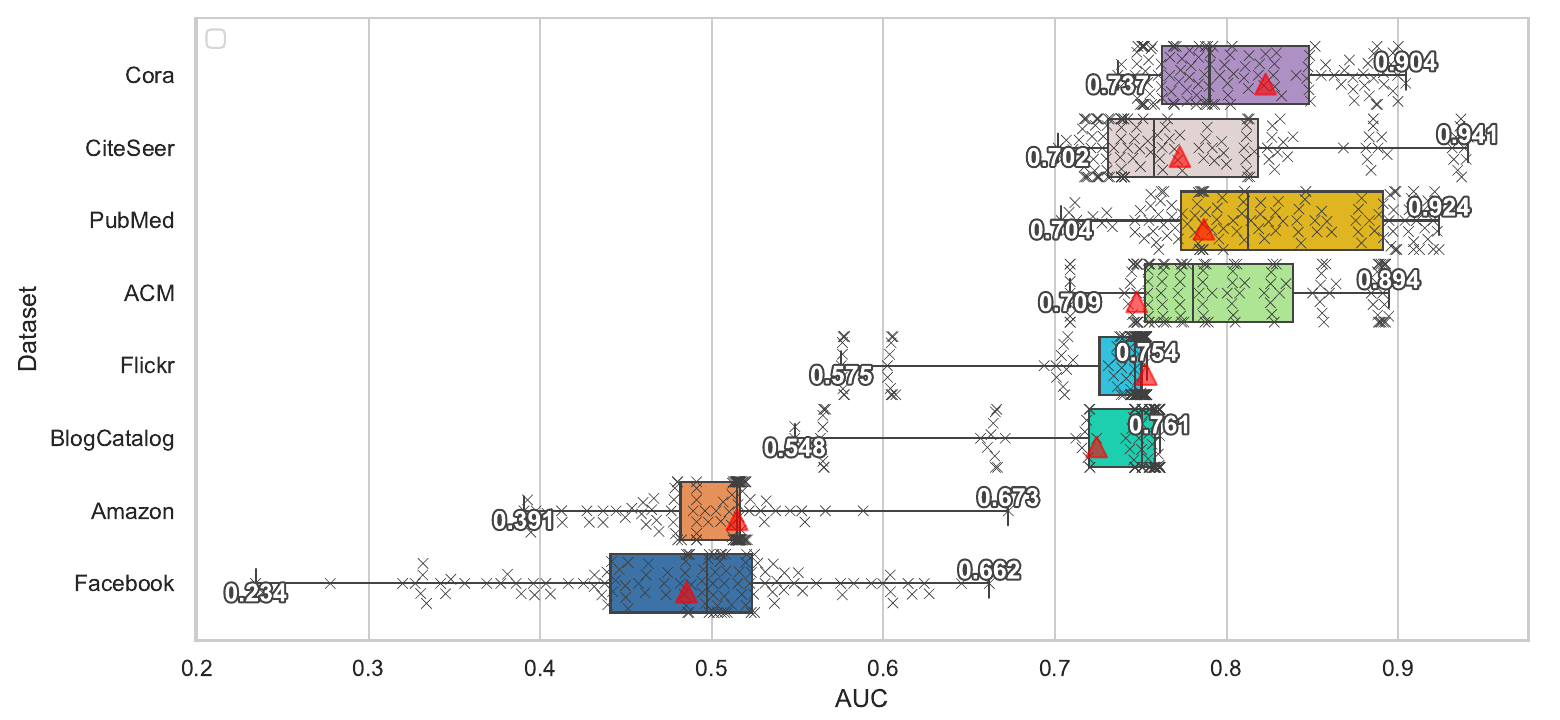}
\caption{Performance variations over different HP configurations for AnomalyDAE \citep{fan2020anomalydae} on different benchmark datasets. }
\label{Fig:ROC_Var_AnomalyDAE}
\end{figure}

\begin{figure}
\centering
\includegraphics[width=13cm]{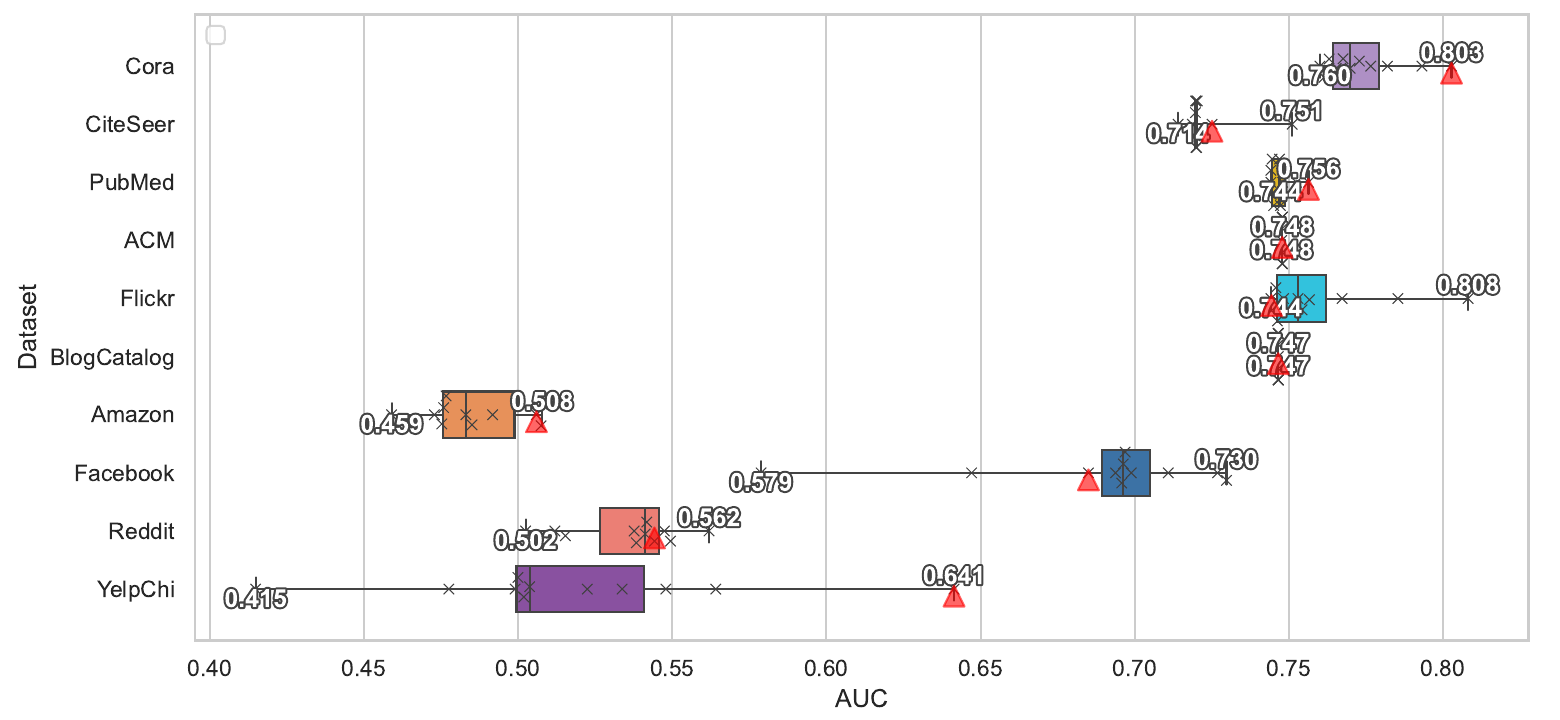}
\caption{Performance variations over different HP configurations for GUIDE \citep{yuan2021higher} on different benchmark datasets. }
\label{Fig:ROC_Var_GUIDE}
\end{figure}

\begin{figure}
\centering
\includegraphics[width=13cm]{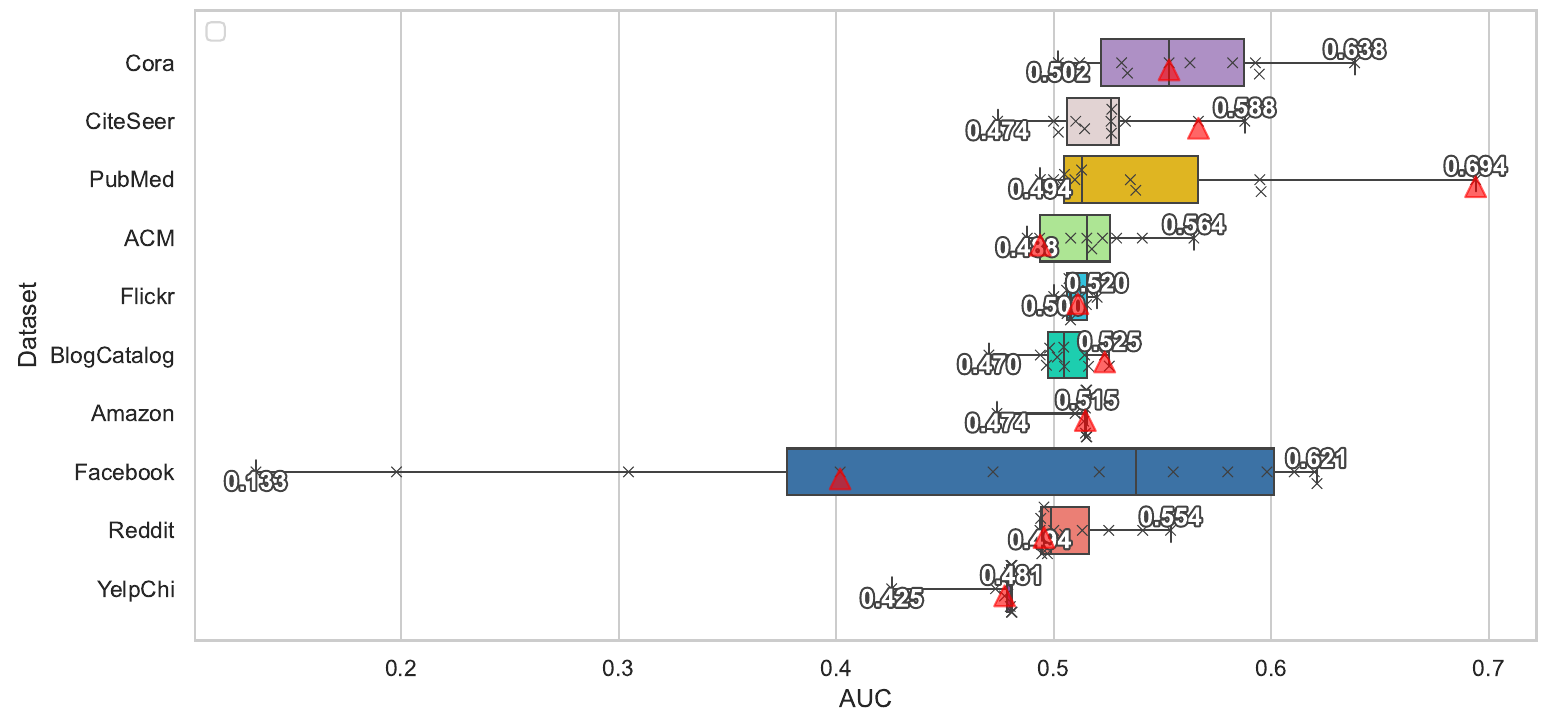}
\caption{Performance variations over different HP configurations for GAAN \citep{chen2020generative} on different benchmark datasets.
}
\label{Fig:ROC_Var_GAAN}
\end{figure}

\section{Similar Observations in Other Papers}
\label{app:SimilarObservations}
\cite{liu2022bond} conduct 
a comprehensive benchmark for unsupervised graph anomaly detection. From their results (note that their experiment setting is slightly different from ours), we can have similar observations as follows by comparing the average AUC vs max AUC:
\begin{itemize}
    \item \textbf{Radar} \citep{li2017radar} is not sensitive to hyper-parameters (0.65 VS 0.66 on Cora, 0.99 VS 0.99 on Weibo, 0.55 VS 0.57 on Reddit, 0.52 VS 0.52 on Disney, 0.53 VS 0.53 on Books), but it will suffer from OOM errors for large graphs;
    \item \textbf{ANOMALOUS} \citep{peng2018anomalous} is very sensitive to hyper-parameters on some datasets (\underline{0.55 VS 0.68 on Cora}, 0.99 VS 0.99 on Weibo, \underline{0.55 VS 0.60 on Reddit}, 0.52 VS 0.52 on Disney, 0.53 VS 0.53 on Books), and it will suffer from OOM errors for large graphs; 
    \item  \textbf{DOMINANT} \citep{ding2019deep} is very sensitive to hyper-parameters on some datasets (0.83 VS 0.84 on Cora, \underline{0.76 VS 0.85 on Flickr}, \underline{0.85 VS 0.93 on Weibo}, \underline{0.50 VS 0.58 on Books}, 0.56 VS 0.56 on Reddit, \underline{0.47 VS 0.55 on Disney})
    \item  \textbf{AnomalyDAE} \citep{fan2020anomalydae} is very sensitive to hyper-parameters on some datasets (0.83 VS 0.85 on Cora, \underline{0.86 VS 0.91} on Amazon, 0.66 VS 0.70 on Flickr, 0.91 VS 0.93 on Weibo, 0.56 VS 0.56 on Reddit, \underline{0.49 VS 0.55 on Disney}, \underline{0.54 VS 0.69 on Books});
    \item \textbf{GUIDE} \citep{yuan2021higher} is very sensitive to hyper-parameters on some datasets (\underline{0.39 VS 0.53 on Disney}, \underline{0.52 VS 0.63 on Books}, 0.75 VS 0.78 on Cora), and it will suffer from OOM errors on large graph (including Amazon, Flickr, Weibo, Reddit). It needs much time and memory for training as it employs a graph motif counting algorithm to extract structural information;
    \item \textbf{CONAD} \citep{xu2022contrastive} is very sensitive to hyper-parameters on some datasets (\underline{0.79 VS 0.84 on Cora}, 0.81 VS 0.82 on Amazon, 0.65 VS 0.67 on Flickr, \underline{0.85 VS 0.93 on Weibo}, 0.56 VS 0.56 on Reddit, \underline{0.48 VS 0.53 on Disney}, \underline{0.52 VS 0.63 on Books}). 
\end{itemize}

\section{Summary of existing SSL-based graph anomaly detection methods}
Existing SSL-based graph anomaly detection methods are summarized in Table~\ref{tab:algo_summary}, which includes the datasets used to test, the core principles of SSL techniques, the involved hyper-parameters (only SSL related ones), and their public implementations.

\begin{table*}
  \caption{SSL-related HPs for different algorithms, where ``\textbf{Range}" indicates the tested values in grid search. 
  }
  \label{tab:HP_search}
  \centering
  \resizebox{\textwidth}{!}{
  \begin{tabular}{p{5cm}p{2cm}p{9cm}}
    \toprule
    \textbf{Algo} & \textbf{HPs} & \textbf{Range} \\
    \midrule
    \multirow{2}{*}{ANEMONE \citep{jin2021anemone}} & $K$ & \{2, 3, 4, 5\}\\
    & $\alpha$ & \{0, 0.01, 0.1, 0.2, 0.3, 0.4, 0.5, 0.6, 0.7, 0.8, 0.9, 0.99, 1\}\\
    \hline
    \multirow{3}{*}{AnomalyDAE \cite{fan2020anomalydae}} & $\alpha$& \{0.01, 0.1, 0.2, 0.3, 0.4, 0.5, 0.6, 0.7, 0.8, 0.9, 0.99, 1\}\\
    & $\eta$& \{1, 2, 3, 4, 5, 6, 7, 8, 9, 10\}\\
    & $\theta$& \{10\}\\
     \hline
     CoLA \citep{liu2021anomaly} & $K$ & \{2, 3, 4, 5\}\\
    \hline
    \multirow{8}{*}{CONAD \citep{zhang2022reconstruction}} & $r$& \{0.10\}\\
    & $p1$& \{0.25\}\\
    & $p2$& \{0.25\}\\
    & $p3$& \{0.25\}\\
    & $p4$& \{0.25\}\\
    & $m$& \{0.5\}\\
    & $\lambda$&\{0.01, 0.1, 0.2, 0.3, 0.4, 0.5, 0.6, 0.7, 0.8, 0.9, 0.99, 1\}\\
    & $\eta$&\{0.01, 0.5, 0.99, 1\}\\
     \hline
     DOMINANT \citep{ding2019deep} & $\alpha$  & \{0.01, 0.1, 0.2, 0.3, 0.4, 0.5, 0.6, 0.7, 0.8, 0.9, 0.99, 1\} \\
     \hline
    \multirow{1}{*}{GAAN \citep{chen2020generative}}
    &$\alpha$& \{0, 0.01, 0.1, 0.2, 0.3, 0.4, 0.5, 0.6, 0.7, 0.8, 0.9, 0.99, 1\}\\
     \hline
    \multirow{4}{*}{GRADATE \citep{duan2023graph}} & $P$& \{0.20\}\\
    & $\alpha$& \{0.9\}\\
    & $\beta$& \{0, 0.01, 0.1, 0.2, 0.3, 0.4, 0.5, 0.6, 0.7, 0.8, 0.9, 0.99, 1\}\\
    & $\gamma$& \{0, 0.01, 0.1, 0.2, 0.3, 0.4, 0.5, 0.6, 0.7, 0.8, 0.9, 0.99, 1\}\\
    \hline
    \multirow{2}{*}{GUIDE \citep{yuan2021higher}} & $D$& \{4\}\\
    &$\alpha$& \{0.01, 0.1, 0.2, 0.3, 0.4, 0.5, 0.6, 0.7, 0.8, 0.9, 0.99\}\\
    \hline
    \multirow{3}{*}{SL-GAD \citep{zheng2021generative}} & $K$& \{2, 3, 4, 5, 6, 7, 8, 9\}\\
    & $\alpha$& \{0.01, 0.1, 0.2, 0.3, 0.4, 0.5, 0.6, 0.7, 0.8, 0.9, 0.99, 1\}\\
    & $\beta$& \{0.6\}\\
     \hline
    \multirow{3}{*}{Sub-CR \citep{zhang2022reconstruction}} & $K$&  \{2, 3, 4, 5, 6, 7, 8, 9\}\\
    & $\alpha$& \{0.01\}\\
    & $\gamma$& \{0.01, 0.1, 0.2, 0.3, 0.4, 0.5, 0.6, 0.7, 0.8, 0.9, 0.99, 1\}\\
  \bottomrule
\end{tabular}
}
\end{table*}

\begin{table*}[]
    \caption{Summary of existing SSL-based graph anomaly detection methods.}
    \label{tab:algo_summary}
    \centering
    \small
    \resizebox{\textwidth}{!}{
    \begin{tabular}{p{4.5cm}p{1.5cm}p{3.5cm}p{3cm}p{3.6cm}p{1cm}}
        \toprule
        Method & Venue & Datasets & SSL Methods & Hyperparameters & Code \\
        \midrule
        ANEMONE \citep{jin2021anemone} & CIKM'21 & Cora, Citeseer, PubMed & Node-Node CL, Node-Sub CL & Ego-Net size ($K$), Combination weights & \href{https://github.com/GRAND-Lab/ANEMONE}{{\color{blue}Github}}\\
        AnomalyDAE \citep{fan2020anomalydae} & ICASSP'20 & ACM, Flickr, BlogCatalog & Attribute Recon, Structure Recon & Penalty HPs, Combination weights & PyGOD\\
        CoLA \citep{liu2021anomaly} & TNNLS'21 & Cora, Citeseer, Pubmed, BlogCatalog, Flickr, ACM, ogbn-arxiv & Node-Sub CL & Random walk length ($K$) & \href{https://github.com/GRAND-Lab/CoLA}{{\color{blue}Github}}, PyGOD\\
        CONAD \citep{xu2022contrastive} & PAKDD'22 & Amazon, Flickr, Enron, Facebook, Twitter & Node-Sub CL, Attribute Recon, Structure Recon & Augmentation sampling probabilities, combination weights & PyGOD, \href{https://github.com/zhiming-xu/conad}{{\color{black}Github}}\\
        DOMINANT \citep{ding2019deep} & ICDM'19 & ACM, Flickr, BlogCatalog & Attribute Recon, Structure Recon & Combination weight & PyGOD\\
        GAAN \citep{chen2020generative} & CIKM'20 & ACM, Flickr, BlogCatalog & Attribute Recon, Discr Loss & Combination weight & PyGOD\\
        GRADATE \citep{duan2023graph} & AAAI'23 & EAT, WebKB, UAT, Cora, UAI2010, Citation & Node-Node CL, Node-Sub CL, Sub-Sub CL & Proportion of modified edges ($P$), Combination weights & \href{https://github.com/FelixDJC/GRADATE}{{\color{blue}Github}}\\
        GUIDE \citep{yuan2021higher} & BigData'21 & Cora, Citation, PubMed, ACM, DBLP & Attribute Recon, Structure Recon & Combination weight & PyGOD\\
        SL-GAD \citep{zheng2021generative} & TKDE'21 & Cora, Citeseer, PubMed, ACM, Flickr, BlogCatalog & Node-Sub CL, $~~~$Attribute Recon & Random walk length ($K$), Combination weights & \href{https://github.com/KimMeen/SL-GAD}{{\color{blue}Github}}\\
        Sub-CR \citep{zhang2022reconstruction} & IJCAI'22 & Cora, Citeseer, PubMed, Flickr, BlogCatalog & Node-Sub CL, $~~~$Attribute Recon & Random walk length ($K$), Teleport probability $\alpha$, Combination weights & \href{https://github.com/Zjer12/Sub}{{\color{blue}Github}}\\
        \bottomrule
    \end{tabular}
    }
\end{table*}

\section{Search Space Approximation based on SMBO}
\label{appendix:SMBO}

\subsection{Performance Surrogate Functions} Although discretization of continuous domains can largely reduce the search space, it is still computationally prohibitive to search the full discretized HP space when the number of HPs is large. Therefore, we learn a regressor $g(\cdot)$ which aims to to learn the mapping from HP settings onto the performance metric (namely the domain of $T(\cdot)$). Note that $g(\cdot)$ should be different for different combinations of graph and graph anomaly detector $[\mathcal{G},f(\cdot)]$, and we call these functions \textit{performance surrogate functions}. Gaussian Process (GP) \citep{williams1995gaussian} is one popular choice for $g(\cdot)$. Based on these \textit{performance surrogate functions}, we can identify promising HPs without running experiments on all possible HPs, which will be illustrated in next subsection.

\subsection{SMBO-based Optimization}
Particularly, we leverage Sequential Model-based Optimization (SMBO) \citep{jones1998efficient} to iteratively and efficiently identify promising HP configurations to evaluate, and finally output the optimal one as follows.  Similar idea is also explored in \cite{zhao2022toward}.

\textbf{Initialization} Specifically, we first randomly sample a small number of HPs $\boldsymbol{\lambda}_{eval} = \{\boldsymbol{\lambda}_{1},\boldsymbol{\lambda}_{2},..., \boldsymbol{\lambda}_{J}\}$ with $J \ll M$. Second, for each HP, we compute its unsupervised performance metric score ${t}(\mathcal{G})$, leading to pairs $\{(\boldsymbol{\lambda}_{1},{t}_{1}(\mathcal{G})),(\boldsymbol{\lambda}_{2},{t}_{2}(\mathcal{G})), $..., $(\boldsymbol{\lambda}_{J},{t}_{J}(\mathcal{G}))\}$. Third, we employ these pairs to train a specific \textit{performance surrogate function} $g(\cdot)$.

\textbf{Iteration} For each iteration, we leverage $g(\cdot)$ to predict the performance for a sampled HP $\boldsymbol{\lambda}_{j}$, denoted as $\eta_{j}=g(\boldsymbol{\lambda}_{j})$. Moreover, we also utilize  $g(\cdot)$ to predict the uncertainty around the prediction of $\boldsymbol{\lambda}_{j}$, denoted as $\sigma_{j}=\sigma[g(\boldsymbol{\lambda}_{l}\vert\boldsymbol{\lambda}_{l} \in \boldsymbol{\lambda}_{sample})]$. Note that $\boldsymbol{\lambda}_{sample}$ is different from $\boldsymbol{\lambda}_{eval}$, and it is a finite number of HPs that is randomly sampled from the full HP space before discretization. Next, we utilize a so-called \textit{acquisition function} $h(\cdot)$, which can make a trade-off between predicted performance and uncertainty, to select the most promising HP to evaluate. Particularly, we leverage Expected Improvement (EI) \citep{jones1998efficient} as the \textit{acquisition function} since it has shown prominent performances in many studies \citep{zhao2022towards}. Under the mild Gaussian assumption, the EI value of HP setting $\boldsymbol{\lambda}_{j}$ has the following closed-form expression:
\begin{equation}
EI(g(\boldsymbol{\lambda}_{j}))=\left[\phi(\hat{\eta}_{j})+\hat{\eta}_{j}\cdot\Phi(\hat{\eta}_{j})\right]\sigma_{j},
\end{equation}
where $\hat{\eta}_{j} = \frac{\eta_{j}-\eta_{eval}^{*}}{\sigma_{j}}$ if $\sigma_{j}>0$ and $\hat{\eta}_{j}=0$ otherwise. Moreover, $\phi(\cdot)$ and $\Phi(\cdot)$ denote the probability density function and the cumulative distribution function of standard Gaussian distribution, respectively. In addition, $\eta_{eval}^{*}$ is the highest prediction performance on $\boldsymbol{\lambda}_{eval}$ so far. For each iteration, the most promising HP can be obtained as follows:

\begin{equation}
\label{Equ:Promising}
    \boldsymbol{\lambda}^{*} = \underset{\boldsymbol{\lambda}_{j} \in \boldsymbol{\lambda}_{sample}}{\arg\max} h(g(\boldsymbol{\lambda}_{j})),
\end{equation}
where $g(\cdot)=g^{(current)}(\cdot)$ is the surrogate function in the current iteration, which can output the most promising HP $\boldsymbol{\lambda}^{*}$ to evaluate. On this basis, we apply $f(\boldsymbol{\lambda}^{*})$ on graph $\mathcal{G}$ to obtain a vector of anomaly scores $\boldsymbol{s}^{*}$, followed by inputting $\boldsymbol{s}^{*}$ into Equation~\ref{Equ:ImprovedCSM} to obtain the performance metric score $t^{*}$. At last, we update the evaluation HP set as $\boldsymbol{\lambda}_{eval} = \boldsymbol{\lambda}_{eval} \cup \boldsymbol{\lambda}^{*}$, and retrain $g(\cdot)$ with the updated  pairs $\{(\boldsymbol{\lambda}_{1},{t}_{1}(\mathcal{G})),(\boldsymbol{\lambda}_{2},{t}_{2}(\mathcal{G})), $..., $(\boldsymbol{\lambda}_{J},{t}_{J}(\mathcal{G}))\}...,(\boldsymbol{\lambda}^{*},t^{*})\}$. Additionally, we update $\eta_{eval}^{*}$ using the updated $\boldsymbol{\lambda}_{eval}$. 

\section{AutoGAD for Selecting Heterogeneous Anomaly Detectors}
\label{appendix:AutoGAD4ADSelection}
\begin{figure}
    \centering
    \includegraphics[width=1\linewidth]{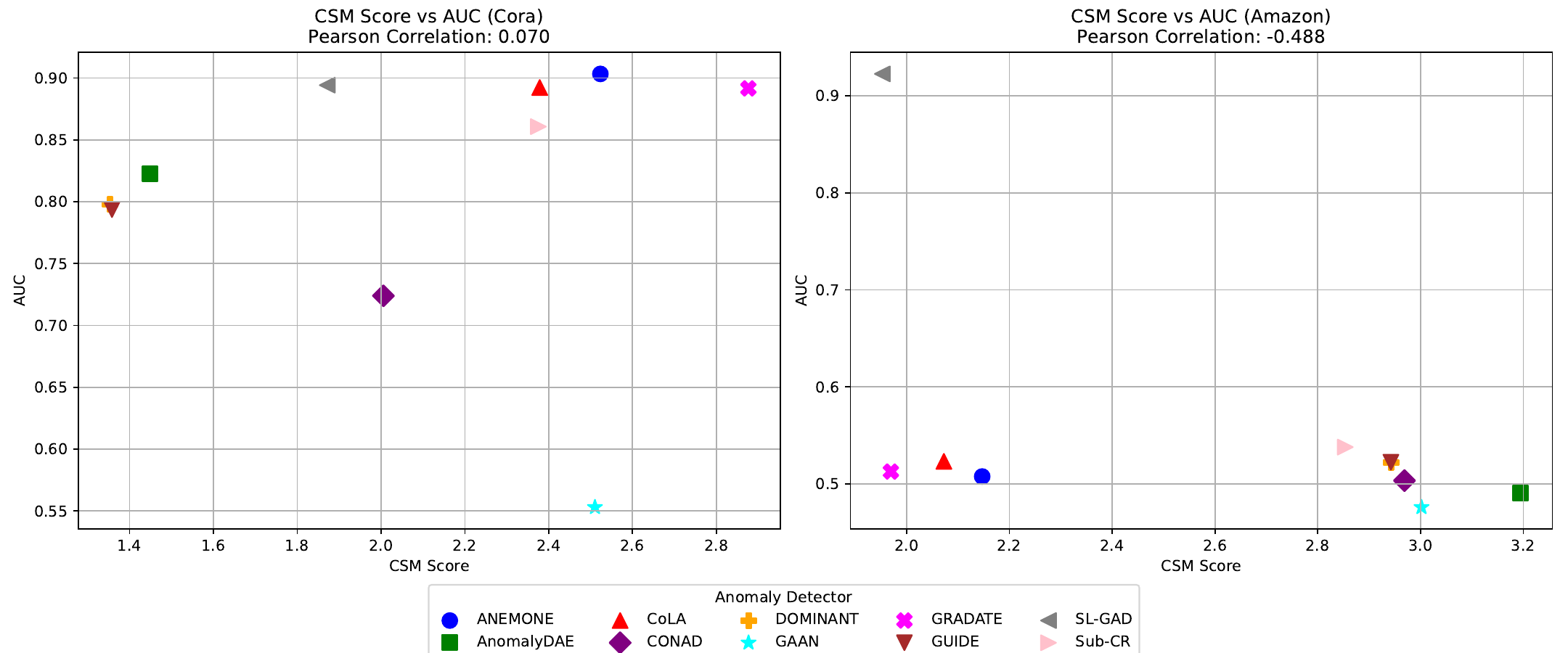}
    \caption{ {\color{black}Performance of AutoGAD in selecting heterogeneous anomaly detectors on selected datasets (results on other datasets are similar and thus omitted).}}
    \label{fig:ADSelection}
\end{figure}

{\color{black}
To evaluate the effectiveness of AutoGAD in selecting heterogeneous anomaly detectors, we compute the Pearson Correlation Coefficient between the highest improved CSM scores (based on Eq.~\ref{Equ:ImprovedCSM}) and the corresponding AUC scores for all anomaly detectors on  each individual dataset.

As shown in Figure~\ref{fig:ADSelection}, the results reveal that AutoGAD's CSM score does not effectively predict the true performance (AUC) of heterogeneous anomaly detectors. Specifically, on the Cora dataset, the Pearson correlation is very weak (0.070), indicating almost no relationship between the CSM score and AUC. On the Amazon dataset, the correlation is negative (-0.488), suggesting that higher CSM scores are, in fact, associated with lower AUC values in many cases. This weak or inverse correlation demonstrates that AutoGAD's scoring mechanism may not be suitable for selecting the best-performing anomaly detectors, as it fails to consistently align with true detector performance. Notably, detectors such as SL-GAD, which achieve high AUC, do not consistently receive high CSM scores, further underscoring the discrepancy. In summary, these findings suggest that AutoGAD's current approach to ranking anomaly detectors is unreliable and may require significant revisions to improve its predictive accuracy.}

\end{document}